%% file: main.tex
\documentclass[accepted]{uai2026}
                        
\usepackage[british]{babel}
\usepackage[sort]{natbib}
\usepackage{mathtools}
\usepackage{booktabs}
\usepackage{subcaption}
\usepackage{dblfloatfix}
\usepackage{tikz}
\usepackage{amssymb}
\usepackage{amsthm}
\usepackage{algorithm}
\usepackage{algpseudocode}
\usepackage{cleveref}

\input{def}

\makeatletter
\renewcommand{\notice@text}{\textit{Accepted for \UAI@long} (\UAI@short).}
\makeatother

\title{Bayesian Experimental Design via Score Matching}

\author[1]{\href{mailto:<angus.phillips@stats.ox.ac.uk>?Subject=Your UAI 2026 paper}{Angus Phillips}{}}
\author[1]{Gavin Kerrigan}
\author[1]{Tom Rainforth}
\affil[1]{%
    Department of Statistics\\
    University of Oxford
}
  
\begin{document}
\maketitle

\begin{abstract}
      Policy-based approaches to Bayesian experimental design (BED) allow the learning of deep policy networks that adaptively make intelligent design decisions based on previously collected data. However, the training of such policies is often held back by a fundamental challenge: the double intractability of the expected information gain (EIG). This necessitates expensive or complex approximations that restrict the effort one can invest in optimising the policy itself. To address this, we show that the double intractability of the EIG can be isolated from the policy learning by first solving a score matching problem that is independent of the policy used, then using the learned score approximation to train the policy in a singly intractable manner. This turns the key multiplicative cost into an additive one and reduces the computational burden on the policy training itself, making it far cheaper to train the policy multiple times when needed, e.g. for architecture search, hyperparameter tuning, or avoiding local optima. In our experiments we train multiple competitive policies without inducing a multiplicative cost in likelihood evaluations, which can increase performance by allowing us to select the best policy even without performing hyperparameter or architecture searches.   
\end{abstract}

\vspace{-6pt}
\section{Introduction}
\vspace{-2pt}

\looseness=-1
The optimal design of experiments is critical for conducting informative and cost-effective scientific experiments in any field. Bayesian Experimental Design (BED) \citep{lindley_measure_1956, chaloner_bayesian_1995,rainforth_modern_2024,huan_optimal_2024} offers a theoretically principled approach to selecting experimental designs grounded in Bayesian decision theory that has found successful application in fields as broad as physics \citep{dushenko_sequential_2020, mcmichael_sequential_2021}, psychology \citep{watson_quest_2017}, quantum computing \citep{sarra_deep_2023}, machine learning \citep{gal_deep_2017} and beyond. Typically, BED is performed by maximising the expected information gain (EIG) of the experiment(s)~\citep{lindley_measure_1956}, where the information gain is given by the reduction in Shannon entropy \citep{shannon_mathematical_1948} from our prior to posterior beliefs in some target quantities of interest, and the expectation is over possible data we might see.

Of particular importance in many applications is the ability to propose the next design in real time, such as conducting a live survey with human participants \citep{pasek_optimizing_2010,vincent_darc_2017} or learning via control of a live dynamical system \citep{iollo_bayesian_2025}. Unfortunately, this can be difficult in practice, as the EIG is a doubly intractable quantity that is very challenging to optimise~\citep{rainforth_nesting_2018,rainforth_modern_2024}, while the need to repeatedly do inference between iterations can also be problematic.

\looseness=-1
Policy-based BED \citep{foster_deep_2021,shen_bayesian_2023,hedman_stepdad_2025,ivanova_implicit_2021,blau_optimizing_2022,iqbal_nesting_2024,bracher_jadai_2025} offers the most viable approach to overcoming this, by amortising adaptive design decisions into a pre-trained policy which can be deployed in a single forward pass. However, the problem of optimising a doubly intractable objective is now transferred to the policy training itself, where the challenges can actually be compounded due to the complexity of policy learning, which often forms a significantly harder optimisation than just optimising designs~\citep{blau_optimizing_2022}.  Furthermore, variational approaches to side-stepping the double intractability~\citep{foster_variational_2019,foster_unified_2020,kleinegesse_efficient_2019} are often challenging to apply due to the difficulty of learning accurate variational distributions in high-dimensional settings. Existing approaches often thus simply resort to nested or contrastive sampling of the EIG gradient, accepting the substantial computational costs this induces.

To address this, we show that it is possible to disentangle the double intractability of the EIG from the training of a policy, utilising the fact that the information gain itself depends on the policy only through the chosen designs and not the policy parameters themselves. Specifically, we derive a reparameterised form of the EIG gradient where the double intractability is isolated in two terms with no direct dependency on the policy: the Stein score and Fisher score of the marginal likelihood given the full design rollout. We then use this to introduce a score matching approach to BED, which we call \scorebed. \scorebed first learns a single network which approximates both of these scores using an objective we call marginal score matching, then substitutes the learned score network into our reparameterised EIG gradient, yielding a singly intractable estimator that can be used to train the policy via stochastic gradient descent.

A key benefit of our two-stage setup is that it transforms the typically multiplicative cost of nested sampling procedures into an additive cost of two singly intractable problems, allowing gradient-based policy training without expensive nested sampling.  Critically, the approximations we learn are entirely independent of any policy and thus the same upfront approximation can be reused repeatedly to allow sample-efficient and scalable training for any number of policies. This provides substantial benefits to policy training by allowing cheaper hyperparameter and architecture searches, longer policy training, or multiple retrainings to avoid local optima, all of which can be important in practice due to the difficulty of policy training~\citep{ivanova_implicit_2021}. 

Experimentally, we find that, when matching total budgets, \scorebed provides competitive performance compared to existing approaches that simultaneously deal with double intractability of the EIG and policy training. Moreover, as only a small proportion of the budget in \scorebed is on the policy training itself, it can cheaply train multiple such competitive policies, thereby paving the way for more nuanced and computationally intensive policy training schemes.

\vspace{-3pt}
\section{Background}
\vspace{-2pt}

\textbf{Bayesian Experimental Design.}~ In BED, one assumes a Bayesian model for an experiment consisting of prior $p(\theta)$ and likelihood $p(y \mid \theta, \xi)$ where $\xi$ denotes the design, $y$ is the experiment outcome, and $\theta$ are the target variables we wish to learn about. The goal is to select $\xi$ to maximise some expected utility function under the data generating process. The standard choice of utility function in BED literature is the information gain on model parameters, defined by \citet{lindley_measure_1956} as the reduction in Shannon's entropy from the prior to the posterior, $\IG_\theta(y, \xi) = \mathbb{H}[p(\theta)] - \mathbb{H}[p(\theta\mid y, \xi)]$. Taking expectation under the data generating process yields the EIG objective $\mathcal{I}(\xi) = \E_{p(y\mid \xi)}\{\IG_\theta(y, \xi)\}$. 

We consider a sequential setting with $T$ rounds of experiments, therefore let $h_t = \{y_s, \xi_s\}_{s=1}^t$ denote the experimental history to time $t$. Following \citet{foster_deep_2021}, we learn an adaptive policy network $\pi_\phi(\xi_t\mid h_{t-1})$ which is trained to propose the optimal next design given $h_{t-1}$ on the fly. There are several advantages of such a policy-based approach \citep{foster_deep_2021, rainforth_modern_2024}, including significantly faster deployment times versus greedy approaches, avoiding the need for potentially inaccurate inference at each time step, and making non-myopic decisions by optimising up front for the total information gain over the entire experiment. In this policy-based setting, we will write the data generating process under the policy $\pi_\phi$\footnote{When $\pi_\phi(h_{t-1})$ is deterministic, this distribution becomes degenerate. We ignore the measure-theoretic considerations in this case and continue to write $p(y\mid\xi)$ even when $\xi$ is deterministic.} as $p_\phi(\y, \x) = \int p(\theta) p_\phi(\y, \x \mid \theta) \rm{d}\theta$ where
\begin{equation*}
    p_\phi(\y, \x\mid\theta) =\prod_{t=1}^T \pi_\phi(\xi_t \mid h_{t-1}) p(y_t\mid\theta, \xi_t, h_{t-1}).
\end{equation*}
\citet{foster_deep_2021} propose optimising the total expected information gain over the whole experiment, therefore giving the policy-based EIG objective:
\begin{equation} 
    \mathcal{I}_T(\phi) = \E_{p_\phi(y_{1:T}, \xi_{1:T})}\{\IG_\theta(y_{1:T}, \xi_{1:T})\},
\end{equation}
where $\IG_\theta(y_{1:T}, \xi_{1:T}) = \mathbb{H}[p(\theta)] - \mathbb{H}[p(\theta\mid y_{1:T}, \xi_{1:T})]$. 

Policy-based BED is therefore concerned with finding the optimal policy $\phi^* = \argmax_\phi \mathcal{I}_T(\phi)$. We take a stochastic gradient-based approach to this problem, although previous works have also searched for designs using Bayesian optimisation~\citep{snoek_practical_2012}, such as~\citet{kleinegesse_efficient_2019, foster_variational_2019}. %This does preclude us from considering problems with discrete design spaces. 

\textbf{Approximating the EIG.}~ A key challenge in BED comes from the double intractability of the objective. The EIG can be written either as an expected KL divergence between the posterior and prior: 
\begin{equation*}
    \mathbb{E}_{p(\theta) p_\phi(\y, \x\mid \theta)}\{\log p(\theta\mid \y, \x) - \log p(\theta)\},
\end{equation*} 
or as the mutual information between $\y$ and $\theta$ given $\x$: \begin{equation*}
\label{eqn:mi_eig}
\mathbb{E}_{p(\theta) p_\phi(\y, \x\mid \theta)}\left\{\log \frac{p(\y\mid \theta, \x)}{p(\y \mid  \x)}\right\}.
\end{equation*}
In both cases it is clear the objective is doubly intractable as it involves evaluating the posterior or marginal distributions which appear non-linearly within the outer expectation.

\citet{foster_unified_2020} introduce the Prior Contrastive Estimation (PCE) bound which lower bounds the EIG. This bound uses $M$ contrastive samples from the prior to approximate the intractable marginal distribution:
\begin{equation*}
    \mathcal{I}_T(\x) \geq \mathbb{E}\left\{\log \frac{p(\y\mid \theta_0, \x)}{\frac{1}{M+1}\sum_{m=0}^M p(\y\mid \theta_m, \x)}\right\}.
\end{equation*}
\looseness=-1
The inclusion of $\theta_0$ in the estimate of the intractable marginal reduces the variance of the estimate compared with ordinary nested Monte Carlo (NMC). The PCE estimator can be differentiated using a reparameterisation trick~\citep{kingma_autoencoding_2013} or a REINFORCE gradient estimator~\citep{williams_simple_1992a}, allowing gradient-based optimisation. 

An alternative approach is to construct a functional or variational approximation of the intractability in the EIG. \citet{foster_variational_2019} propose replacing the intractable posterior with a variational approximation $q(\theta\mid \y, \x)$, which they show results in a lower bound of the EIG due to the Barber-Agakov bound \citep{barber_information_2003}. Alternatively, they show that using a variational approximation of the intractable marginal $q(\y\mid \x)$ leads to an upper bound on the EIG. \citet{foster_unified_2020} further propose a unified procedure which co-optimises the variational approximation and the designs, meaning the approximation is specialised to the particular policy network.

\textbf{Score Matching.}~ Score matching \citep{hyvarinen_estimation_2005} refers to techniques used to approximate the Stein score $\nabla_z \log p(z)$ of a distribution with density $p(z)$ by minimising the explicit score matching objective $\frac{1}{2}\E_z\{\|s_\psi(z) - \nabla \log p(z)\|^2\}$. Being typically intractable, the target score is unavailable to us, therefore equivalent tractable objectives have been proposed \citep{hyvarinen_estimation_2005, vincent_connection_2011, song_sliced_2019} such as the sliced score matching objective:
\begin{equation}
    \mathcal{J}_{\text{SSM}}(\psi) = \mathbb{E}_{\nu} \mathbb{E}_{z}\big\{v^T \nabla_z s_\psi(z)v + \frac{1}{2}(v^T s_\psi(z))^2\big\}
\end{equation}
which only requires samples of $z$.

Score matching has achieved particular success in denoising diffusion generative models \citep{song_generative_2019, ho_denoising_2020}, where a variant known as denoising score matching \citep{vincent_connection_2011} is used to learn the noise-conditional score function of the generative process. This has proved successful even in high-dimensional and conditional settings \citep{dhariwal_diffusion_2021, rombach_highresolution_2022}. Denoising score matching considers a noisy observation $\tilde{z}$ of $z$, according to the Gaussian convolution $p(\tilde{z}) = \int p(z)p(\tilde{z}\mid z)\rmd z$ where $p(\tilde{z}\mid z)$ is a noising kernel. The score of the noised distribution is shown to satisfy $\nabla_z \log p(\tilde{z}) = \int \nabla \log p(\tilde{z}\mid z) p(z\mid \tilde{z}) \rmd z$ which is the unique minimiser of the regression objective $\mathcal{J}_{\text{DSM}}(\psi) = \mathbb{E}_{z, \tilde{z}}\big\{\|s_\psi(\tilde{z}) - \nabla \log p(\tilde{z}\mid z)\|^2\big\}$. 

\section{Method}
\label{sec:method}

We contribute three significant observations and results in the following section, which culminate in our score-based method for EIG optimisation of policies. Firstly, we observe that the information gain is fundamentally independent of how the observed data was collected, so approximations to deal with the double intractability do not necessarily have to depend on the policy. Secondly, we derive a precise EIG gradient expression that shows the exact source of the double intractability which we aim to resolve in a way that is detached from the policy. Thirdly, we show how these intractabilities can be naturally approximated with a score network trained on a specific score matching loss. Combining these facts leads to our \scorebed approach, where we first learn an accurate and computationally favourable approximation of the EIG gradient upfront, then use this learned approximation for training policies. This transforms the computationally challenging, doubly intractable EIG objective into two sequential singly intractable problems.

The importance of this separation is that it can offer significant computational benefits by replacing a multiplicative cost (number of training steps times number of inner samples) with an additive cost (cost of training the score network plus cost of the resulting \emph{singly intractable} policy training). By substantially reducing the burden on the policy training part of the overall procedure, this gives scope to perform architecture adjustments, hyperparameter tuning, and even multiple training restarts to obtain a good policy network~\citep{rainforth_modern_2024}. Historically this required all the computation of the full procedure to be repeated again; for our score matching approach it only requires the now singly intractable policy training  to be repeated, providing substantial gains. This is particularly important if likelihood evaluations are expensive, and is expected to be increasingly beneficial as the problems we work with require more complex policies and training mechanisms~\citep{blau_optimizing_2022}.

\textbf{Separating intractabilities from policies.}~ 
We begin with an important observation that, to the best of our knowledge, has not previously been noted in the context of policy-based BED: the information gain $\IG_\theta(y_{1:T}, \xi_{1:T}) = \mathbb{H}[p(\theta)] - \mathbb{H}[p(\theta\mid y_{1:T}, \xi_{1:T})]$ itself has no direct dependence on the design policy but is simply a function of the data $(\y, \x)$, i.e.~for given data it is independent of how that data was collected (see Appendix \ref{app:teig} for proof). The policy dependence of the EIG only enters through the outer expectation over data, thus the double intractability of the information gain is completely separable from the policy. The upshot of this is that any approximation of the intractabilities in the information gain has a global optimum that is independent of the policy (assuming our approximator has infinite capacity). We can thus learn approximation schemes upfront that are policy-independent, then apply them as fixed approximations for training any policy.

\textbf{EIG gradient expression.}~
The critical quantity required for gradient-based training of policies is the gradient of the EIG objective with respect to policy parameters, $\nabla_\phi \mathcal{I}_T(\phi)$. Existing approaches typically only consider estimating this quantity indirectly, by differentiating through an estimate or bound of the EIG itself. However, variational bounds may be uninformative if the variational approximation is not accurate and differentiating such an approximation does not guarantee an accurate approximation of its gradient. We therefore sidestep these potential challenges by choosing to approximate the gradient of the EIG directly. 

To that end, we present an exact EIG gradient expression which clearly highlights the exact nested intractabilities that are present, as well as how the policy parameters $\phi$ enter the expression \emph{independently} of those intractabilities. 

\begin{theorem}[Reparameterised EIG Gradient Expression]
\label{prop:eig_grad}
    The following expression for the gradient of the EIG holds:
    \begin{equation}
    \label{eqn:adaptive_design_grad_prop}
        \begin{split}
            \nabla_\phi \mathcal{I}_T(\phi) &= \mathbb{E}_{p(\theta)q(\epsilon)}\bigg\{\frac{\mathrm{d}}{\mathrm{d}\phi} \log p({y}_{1:T}\mid {\xi}_{1:T}, \theta) - \\ &\quad \frac{\partial}{\partial {y}_{1:T}} \log p({y}_{1:T}\mid {\xi}_{1:T}) \frac{\partial {y}_{1:T}}{\partial \phi} - \\ & \quad \frac{\partial}{\partial {\xi}_{1:T}} \log p({y}_{1:T}\mid{\xi}_{1:T}) \frac{\partial {\xi}_{1:T}}{\partial \phi}\bigg\}
        \end{split}
    \end{equation}
    where the data and designs are obtained by reparameterisation, $(\y, \x) = g(\epsilon, \theta, \phi)$, with $g$ the reparameterisation map of the data-generating process and $\epsilon \sim q(\epsilon)$ a set of noise variables independent of $\phi$, so that $(\y, \x) \sim p_\phi(\cdot, \cdot\mid\theta)$. The explicit form of $g$ is deferred to Appendix \ref{app:eig_grad_expression} for brevity. In the specific case of static (batch) designs, the expression simplifies to:
    \begin{equation}
        \begin{split}
            \nabla_{\x} \mathcal{I}_T(\x) &= \mathbb{E}_{p(\theta)q(\epsilon)}\bigg\{\frac{\mathrm{d}}{\mathrm{d} \x} \log p({y}_{1:T}\mid \theta, \x) - \\ & \quad \frac{\partial}{\partial {y}_{1:T}} \log p({y}_{1:T}\mid \x)  \frac{\partial {y}_{1:T}}{\partial \x}\bigg\}
        \end{split}
    \end{equation}
    where again $\y = g(\epsilon, \theta, \x)$ via reparameterisation.
\end{theorem}
\begin{proof}
See Appendix \ref{app:eig_grad_expression}.
\end{proof}

The expression in \Cref{eqn:adaptive_design_grad_prop} now reveals the inner structure of the nested intractabilities and dependence on policy parameters $\phi$ in the EIG gradient. In particular, the intractable terms are the Stein score function $\frac{\partial}{\partial {y}_{1:T}} \log p({y}_{1:T}\mid {\xi}_{1:T})$ and the Fisher score function $\frac{\partial}{\partial {\xi}_{1:T}} \log p({y}_{1:T}\mid{\xi}_{1:T})$. Importantly, these terms are \emph{evaluated} at data and designs $({y}_{1:T}, {\xi}_{1:T})$ which are functions of the policy parameter $\phi$ due to reparameterisation, but the terms themselves are entirely conditionally independent of $\phi$ \emph{given} the designs. 

A key intuition here that has not been highlighted by previous work is that the marginal likelihood term $p(y_{1:T}\mid \xi_{1:T})$ that appears in the EIG and its gradient is not the same as the marginal distribution on data for a given policy $p(y_{1:T};\pi_\phi)$. Thus, even though the data distribution obviously depends on our policy, the actual intractable term we are dealing with is given by $\mathbb{E}_{p(\theta)}[p(y_{1:T}\mid \xi_{1:T},\theta)]$, which does not depend on the policy given the designs.

Additionally,~\Cref{prop:eig_grad} shows analytically that the Fisher score term disappears in the case of static designs. This result reduces bias and variance of our gradient approximation in this case, and could also be used to improve EIG gradient estimation when not using our score-based approach.

\begin{algorithm}[t]
\caption{\scorebed}
\label{alg:scorebed}
\algrenewcommand\algorithmicindent{0.6em}
\begin{algorithmic}[1]
\Statex \textbf{Input:} prior $p(\theta)$, likelihood $p(\y\mid\x,\theta)$, reparameterisation map $(\y, \x) = g(\epsilon, \theta, \phi)$, design sampler $q(\x)$, policies $\{\pi_{\phi_1},\dots,\pi_{\phi_P}\}$
\Statex \textbf{Output:} trained policies $\{\pi_{\phi_1},\dots,\pi_{\phi_P}\}$
\Statex
\Statex \textit{\textbf{Stage 1:} Learn intractable score (policy-independent).}
\For{score budget}
    \State Sample batch $\{\x^b \sim q(\x), \; \theta^b \sim p(\theta), \; \y^b \sim p(\y\mid\x^b,\theta^b)\}_{b=1}^{B}$
    \State $\mathcal{J}_{\text{MSM}}(\psi) \gets \dfrac{1}{B}\sum_{b=1}^{B} \| \score(\y^b,\x^b) - \nabla_{y,\xi}\log p(\y^b\mid\x^b,\theta^b)\|^2$
    \State $\psi \gets \textproc{Update}(\psi,\, \nabla_\psi \mathcal{J}_{\text{MSM}}(\psi))$
\EndFor
\Statex
\Statex \textit{\textbf{Stage 2:} Train each policy with the fixed score $\score$.}
\For{$p = 1,\dots,P$}
    \For{policy budget}
        \State Sample batch $\{\theta^n \sim p(\theta), \; \epsilon^n \sim q(\epsilon)\}_{n=1}^{N}$
        \State $(\y^n,\x^n) \gets g(\epsilon^n,\theta^n,\phi_p)$
        \State $\widehat{\nabla_\phi \mathcal{I}_T}(\phi_p) \gets \sum_{n=1}^{N}\Big[\mathrm{d}_{\phi_p}\log p(\y^n\mid\x^n,\theta^n)$
        \Statex \hspace{3em} $-\, \score(\y^n,\x^n)\,\partial_{\phi_p}(\y^n,\x^n)\Big]/N$  \hspace{0.2em} (eq. \ref{eqn:adaptive_design_grad_prop})
        \State $\phi_p \gets \textproc{Update}(\phi_p,\, \widehat{\nabla_\phi \mathcal{I}_T}(\phi_p))$
    \EndFor
\EndFor
\State \Return $\{\pi_{\phi_1},\dots,\pi_{\phi_P}\}$
\end{algorithmic}
\end{algorithm}

\textbf{BED via score-matching.}~ \Cref{prop:eig_grad} makes clear that we need to estimate or approximate $\frac{\partial}{\partial \y} \log p(\y\mid \x)$ and $\frac{\partial}{\partial \x} \log p(\y\mid \x)$ (whether this is done implicitly or directly) as part of an overall scheme for estimating the EIG gradient. Previous work has either done this through estimates of the marginal likelihood that are recalculated for each new $(y_{1:T},\xi_{1:T})$, or using a variational approximation of either the marginal likelihood or posterior, with gradients then taken of the approximation itself.

We propose to instead learn these terms \emph{directly} using score matching. We argue that score matching is a natural choice in this scenario for a number of reasons. Firstly, it explicitly targets accuracy on the intractable score, which as we have shown is the key quantity in the EIG gradient. While one could instead take any density estimator or variational approximation of the intractable marginal and differentiate its log density to approximate the score, such an approach may not be accurate in practice. Indeed, the score function can be expressed as $\nabla_z \log p(z) = \nabla_z p(z)/p(z)$ and density estimation or variational approximations do not control gradient values. In other words, $p \to \nabla \log p$ is not continuous in the metrics used for density estimation or variational approximation. Secondly, the score function is independent of the normalisation constant and therefore the expressivity of a score network is not limited by the restrictions imposed by normalisation which are found in variational approaches. Finally, score matching is simply a supervised regression problem which scales well with dimension and is not complicated by gradient variance in the same way as variational methods \citep{kingma_autoencoding_2013}.

To guarantee the accuracy of our EIG gradient approximation in practice, we show in Appendix \ref{app:error_bounds} that if the reparameterisation map $\phi \mapsto (\y, \x)$ is $L$-Lipschitz, the $\ell_2$ error of our gradient estimator is bounded above by $L$ times the mean score error. Achieving accurate EIG gradient estimates is therefore simply a question of score network capacity, training budget, and avoiding local optima, noting we have access to unlimited model samples during training. 

\textbf{Marginal score matching.}~
We propose a score matching objective which directly recovers the intractable score functions from the likelihood of the model. Removing dependence on $t$ for ease of notation, we are required to learn $\nabla_{y, \xi} \log p(y\mid \xi)$ where $p(y\mid \xi) = \int p(y\mid \xi, \theta) p(\theta) \mathrm{d}\theta$. A simple score identity is as follows:
\begin{align}
    \nabla_{y, \xi}&\log p(y\mid \xi) = \nabla_{y, \xi} \log \left(\int p(y\mid \xi, \theta) p(\theta) \mathrm{d}\theta\right) \notag \\   
    &= \int \nabla_{y, \xi}\left[ \log p(y\mid \xi, \theta)\right] p(\theta\mid \xi, y) \mathrm{d}\theta \label{eqn:msm_identity}
\end{align}
This identity has been considered by \citet{ao_estimating_2024, brehmer_mining_2020} as the basis for Monte Carlo score approximations, but has not been used for training score networks. This inspires the following regression objective, which we refer to as marginal score matching (MSM):
\begin{equation}
    \label{eqn:msm_loss}
    \mathcal{J}_{\text{MSM}}(\psi) = \mathbb{E}_{\xi, y, \theta}
    \{\| s_\psi(y, \xi) - \nabla_{y, \xi}\log p(y\mid \xi, \theta)\|^2\},
\end{equation}
where the expectation is taken over the joint distribution of the designs $\xi$, observations $y$ and parameters $\theta$, and $s_\psi(y, \xi)$ is a deep neural network with parameters $\psi$ which we will use as an approximation to the intractable score functions. It is easy to show that the parameters $\psi^*$ minimising \cref{eqn:msm_loss} are such that $s_{\psi^*}(y, \xi) = \nabla_{y, \xi}\log p(y\mid \xi) \; \forall (y, \xi); q(y, \xi) > 0$ where $q$ is the distribution used to sample $(y, \xi)$. Critically, the objective requires no direct sampling from the posterior, only the joint on $\xi,y,\theta$. Consequently, our approach is suitable in models with high-dimensional parameter spaces where nested sampling approaches suffer exponential sample complexity with parameter space dimension. Our regression targets live in the design and output spaces which are unaffected in complexity by parameter dimension. 

Moreover, assuming a sufficiently powerful score network architecture, we are free to take the expectation over $\xi$ with respect to \emph{any} distribution that has support over all possible designs, because we are simply regressing from $(y, \xi)$ pairs to scores and $s_{\psi^*}(y,\xi)$ for any given $\xi$ is independent of our distribution on $\xi$. Thus, we do not need to sample according to the policy and can instead sample from the joint $\tilde{q}(\xi, y, \theta) = q(\xi) p(\theta) p(y\mid \xi, \theta)$, where $q(\xi)$ is simply chosen to reflect where we most want the score network to be accurate. In our experiments we found that sampling $\x$ from a simple design distribution (e.g. uniform over the design space) was usually sufficient and maintained independence of the approximation from the policy. We experiment with different choices of $q(\xi)$ in Appendix \ref{app:design_sampler_abl}.

We also point out that we could alternatively train the score network by minimising an implicit score matching loss such as the sliced score matching (SSM) loss \citep{song_sliced_2019}, which does not rely on the model likelihood. We found the MSM loss performed better, however the SSM loss is useful for settings with implicit or expensive likelihoods.

\textbf{Score network architecture.}~ We provide full details of our score network architecture in Appendix \ref{app:method_score}, however we remark on two important points in the following. Firstly, in learning $\nabla_{y, \xi} \log p(y \mid \xi)$ we simultaneously recover each of the required score terms $\frac{\partial}{\partial y} \log p(y\mid \xi)$ and $\frac{\partial}{\partial \xi} \log p(y\mid \xi)$ from the same score network. We found that a shared backbone with separate prediction heads for each score component was the most efficient architecture. Secondly, when observations $y_t$ are conditionally i.i.d. given $(\theta, \xi)$, the marginal likelihood $p(\y\mid \x)$ is exchangeable with respect to pairs $(y_t, \xi_t)$. We encode this permutation equivariance as an inductive bias by using a TNP-style architecture with self-attention \citep{vaswani_attention_2017,pmlr-v162-nguyen22b} and add positional encodings to $(y_t, \xi_t)$ in non-exchangeable cases.

\textbf{Applicability.}~  \scorebed relies on explicit likelihoods with differentiable reparameterisations, however these are typical assumptions in alternative approaches \citep{goda_unbiased_2022,foster_deep_2021,ao_estimating_2024,iollo_bayesian_2025}. Our method also requires continuous design and observation spaces in order for the Stein and Fisher score functions to exist. We highlight that a large number of BED problems satisfy these assumptions, meaning that our method is still a practical BED approach for a wide class of problems. 

\vspace{-4pt}
\section{Related Work}
\vspace{-3pt}

\looseness=-1
The main body of BED literature has revolved around estimating or bounding the EIG itself, for instance contrastive estimation~\citep{foster_deep_2021} and variational approaches (see e.g.~\citet{rainforth_modern_2024} and the references therein). Some recent works have further explored direct EIG gradient estimation. \citet{goda_unbiased_2022} introduce an unbiased estimate of the EIG gradient based on multi-level Monte Carlo~\citep{rhee_unbiased_2015}, which acts as a de-biasing scheme in finite samples, at the cost of increased variance. \citet{ao_estimating_2024} also directly target the EIG gradient using expressions similar to~\cref{eqn:adaptive_design_grad_prop} and \cref{eqn:msm_identity}. Contrary to our amortisation approach, they use MCMC to sample the posterior and estimate~\cref{eqn:msm_identity} via Monte Carlo, leading to a high cost per gradient update when training policies and poor sample complexity in high parameter dimensions. A final direct EIG gradient approach is~\citet{iollo_bayesian_2025}, who use sampling as optimisation to roll the nested sampling problem into a single loop, although they do not train policies.

A closely related family of approaches is the variational methods introduced by \citet{foster_variational_2019,foster_unified_2020} and further adopted by \citet{dong_variational_2025,bracher_jadai_2025,huang_aline_2025}. The original work of \citet{foster_variational_2019} applied a two-stage procedure in the case of static designs, learning the variational approximation up front and using the resulting bound on the EIG to optimise designs. However, subsequent policy-based variational approaches \citep{dong_variational_2025,bracher_jadai_2025,huang_aline_2025} co-train the policy and the variational approximation rather than recognising the variational approximation as a global approximation independent of the policy. Furthermore, the BED literature typically considers a policy-dependent data distribution $p(\y; \pi_\phi)$ and thus considers the posterior to be dependent on the policy $\pi_\phi$. We instead emphasise---as shown in Appendix \ref{app:teig}---that the posterior never depends on the policy, nor more generally on any sampling or observation mechanism~\citep{rubin1976inference} under appropriate assumptions. As such, variational methods can be applied in the same two-stage procedure which we adopt, and we demonstrate this in our experiments in \Cref{sec:experiments}.

\looseness=-1
Concurrent work by~\citet{huang2026efficient} also reduces the cost of policy training, leveraging belief representations from a foundation model pre-trained independently of any policy, thereby relieving policy training from learning representations. They approximate the EIG using PCE, so their cost saving is complementary to ours. Finally,~\citet{kleinegesse_bayesian_2020, ivanova_implicit_2021} take an amortised approach for implicit models using critics to build functional approximations which bound the mutual information, while \citet{hedman_stepdad_2025} consider periodic retraining of the policy during the experiment. 
We also note connections to gradient estimators for implicit models, particularly those leveraging score matching as proposed by \citet{song_sliced_2019}. 
One such work is \citet{lim_ardae_2020}, who approximate the gradient of the entropy of an implicit distribution using denoising autoencoders \citep{vincent_connection_2011}. Alternative entropy gradient estimation techniques can also be derived from Stein's identity~\citep{li_gradient_2018, wen_gradient_2021, shi_spectral_2018}.

\vspace{-5pt}
\section{Experiments}
\label{sec:experiments}
\vspace{-4pt}

We empirically validate \scorebed on a wide selection of common experimental design tasks against several competitive baseline approaches. In each task we train adaptive design policies using \scorebed and baselines and report upper and lower bounds on the attained EIG of the final policy. Our experiments span tasks with varying properties, including non-Markovian likelihoods and high-dimensional parameter spaces. We present four BED tasks in the following sections and include an additional task in Appendix \ref{app:extra_gravimetry}. Full experimental details are included in Appendix \ref{app:exp_details}.

\begin{table*}[!t]
    \centering
    \caption{Results for the source location finding task. Every method used a total NLE budget of $6.144 \times 10^9$, except for PCE$^{\ast}$ which used $10\times$ the budget. Results show mean ± one standard error over 4 policy repeats, bold indicates statistically indistinguishable from the best lower bound at 95\% confidence.}
    \label{tab:lf_results}
    \label{tab:lf_k10d3_results}
    \begin{tabular}{lrrrr}
        \toprule
        & \multicolumn{2}{c}{$d=2, K=2$} & \multicolumn{2}{c}{$d=3, K=10$} \\
        \cmidrule(lr){2-3} \cmidrule(lr){4-5}
        Method & Lower Bound & Upper Bound & Lower Bound & Upper Bound \\
        \midrule
        \scorebed (P=1) & 10.81 $\pm$ 0.06 & 10.95 $\pm$ 0.07 & \textbf{9.45 $\pm$ 0.03} & 9.74 $\pm$ 0.05 \\
        \scorebed (P=5) & 10.98 $\pm$ 0.04 & 11.12 $\pm$ 0.04 & \textbf{9.37 $\pm$ 0.04} & 9.65 $\pm$ 0.07 \\
        PCE (P=1) & 10.12 $\pm$ 0.11 & 10.16 $\pm$ 0.11 & 8.82 $\pm$ 0.02 & 8.97 $\pm$ 0.03 \\
        PCE (P=5) & 9.61 $\pm$ 0.06 & 9.64 $\pm$ 0.06 & 8.49 $\pm$ 0.04 & 8.61 $\pm$ 0.04 \\
        PCE$^{\ast}$ (P=1) & 10.83 $\pm$ 0.03 & 10.92 $\pm$ 0.04 & 9.14 $\pm$ 0.08 & 9.33 $\pm$ 0.09 \\
        JointVarPost (P=1) & \textbf{11.52 $\pm$ 0.05} & 11.78 $\pm$ 0.06 & 8.18 $\pm$ 0.04 & 8.44 $\pm$ 0.05 \\
        JointVarPost (P=5) & 11.25 $\pm$ 0.06 & 11.48 $\pm$ 0.08 & 7.86 $\pm$ 0.04 & 8.00 $\pm$ 0.04 \\
        VarMarg (P=1) & 4.51 $\pm$ 0.08 & 4.51 $\pm$ 0.08 & 3.31 $\pm$ 0.06 & 3.31 $\pm$ 0.06 \\
        VarMarg (P=5) & 5.02 $\pm$ 0.08 & 5.03 $\pm$ 0.08 & 3.57 $\pm$ 0.02 & 3.57 $\pm$ 0.02 \\
        VarPost (P=1) & 11.11 $\pm$ 0.04 & 11.23 $\pm$ 0.05 & 7.08 $\pm$ 0.21 & 7.14 $\pm$ 0.22 \\
        VarPost (P=5) & 11.18 $\pm$ 0.02 & 11.31 $\pm$ 0.03 & 7.37 $\pm$ 0.10 & 7.46 $\pm$ 0.11 \\
        \bottomrule
    \end{tabular}
\end{table*}

\begin{figure*}[t]
    \centering
        \includegraphics[width=0.38\linewidth]{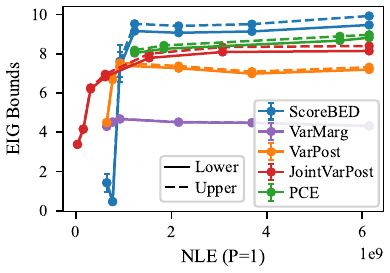}
        \centering ~~~~
        \includegraphics[width=0.38\linewidth]{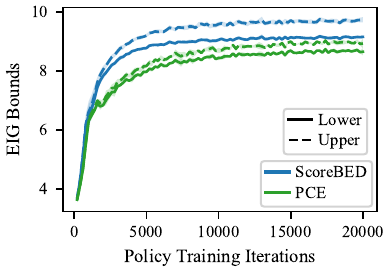}
    \vspace{-3pt}
    \caption{Left: EIG vs NLE ($P=1$) curves for each method on the $d=3,K=10$ location finding setting. Right: Comparing policy training curves for \scorebed and PCE on the same problem.}
    \vspace{-8pt}
    \label{fig:lf_k10d3_curves}
\end{figure*}

\textbf{Procedure.}~ In order to ensure a fair comparison, each task has the same fixed computational budget across methods, specified through the allowed \emph{number of likelihood evaluations} (NLE). For each approach, we obtain a final policy network through two distinct procedures. First, we train a singular policy network and report its achieved EIG bounds with standard errors reporting variation over re-runs of a single policy training. Second, we instead consider training $P$ restarts of the policy network and report the EIG bounds of the best-performing policy of the $P$ candidates (estimated on a held-out test set), reporting standard errors across re-runs of the best-of-$P$ procedure. The configuration of each baseline is adjusted in order to train $P$ policies with the same total NLE budget. Our two-stage approach means that training multiple policies can be done cheaply, therefore the score training budget is only modestly reduced to accommodate increased $P$. On the other hand, approaches which do not amortise the intractabilities of the EIG must reduce their budget per policy proportionally to $P$. In practice, training $P$ policies allows for architecture searches and hyperparameter tuning to increase performance, but we considered a fixed policy configuration to isolate the quality of each method on a consistent policy training task.

\looseness=-1
\textbf{Baselines.}~ We compare \scorebed to several established BED baselines which cover both contrastive sampling and variational approaches. First, we compare against PCE \citep{foster_deep_2021}. In some cases we also include PCE$^\ast$, which uses a $10\times$ larger-budget variant of PCE since related works sometimes consider larger $M$ than our original budgets. Second, we compare against two pre-trained policy-based extensions of the variational approaches of \citet{foster_variational_2019}, the variational marginal (VarMarg) and variational posterior (VarPost), where we pre-train variational approximations of the intractable marginal and posterior distributions, and then train the policy using these fixed approximations (resulting in upper and lower bounds on the EIG respectively). We emphasise that these approaches do not currently appear in the literature as variational approaches are typically co-trained alongside the policy network. Finally, we compare to the co-trained variational posterior approach of \citet{foster_unified_2020} (JointVarPost), in which the policy is co-trained with the variational approximation by maximising the Barber-Agakov lower bound on the EIG. On the dynamical systems tasks, we also include \iosmc \citep{iqbal_nesting_2024}. VarMarg and VarPost represent two-stage approaches which share computation across multiple policies similarly to \scorebed, while PCE, JointVarPost and \iosmc must repeat all computations for each policy. Details of each baseline are included in Appendix \ref{app:method_baselines}. We also tested MLMC \citep{goda_unbiased_2022} but found that the wall-clock cost was prohibitive and the results were not competitive, see Appendix \ref{app:mlmc_results}. 

The exact NLE cost of each method is calculated in Appendix \ref{app:nle_calcs} and the configurations of each baseline to respect the exact budget are detailed in Appendix \ref{app:nle}. In order to accommodate different numbers of policies $P$ while respecting the same total budget, we varied the upfront training budget for \scorebed, VarMarg and VarPost, the total training budget for JointVarPost and the number of inner samples and particles for PCE and \iosmc respectively. 

\vspace{-3pt}
\subsection{Source Location Finding}
\vspace{-2pt}

\begin{table*}[!t]
    \centering
    \caption{Results for dynamical systems. All tasks have a common NLE budget of $1.024 \times 10^{10}$. \scorebed selects policies across three score network seeds each receiving one third of the budget. The NLE cost of \iosmc is random due to extra steps following degeneracy criteria; exact values are presented in Appendix \Cref{tab:nle_iosmc}. Results show mean ± one standard error over 4 policy repeats, bold indicates statistically indistinguishable from the best lower bound at 95\% confidence.}
    \label{tab:dyn}
    \label{tab:sp_results}
    \label{tab:cp_results}
    \label{tab:dp_results}
    \begin{tabular}{lrrrrrr}
        \toprule
        & \multicolumn{2}{c}{Stochastic pendulum} & \multicolumn{2}{c}{Cart-pole} & \multicolumn{2}{c}{Double link} \\
        \cmidrule(lr){2-3} \cmidrule(lr){4-5} \cmidrule(lr){6-7}
        Method & Lower Bound & Upper Bound & Lower Bound & Upper Bound & Lower Bound & Upper Bound \\
        \midrule
        \scorebed (P=3) & \textbf{3.91 $\pm$ 0.02} & 3.91 $\pm$ 0.02 & 6.94 $\pm$ 0.01 & 6.96 $\pm$ 0.00 & 11.01 $\pm$ 0.07 & 11.59 $\pm$ 0.12 \\
        \scorebed (P=50) & \textbf{3.95 $\pm$ 0.02} & 3.95 $\pm$ 0.02 & \textbf{7.17 $\pm$ 0.02} & 7.19 $\pm$ 0.02 & 11.78 $\pm$ 0.01 & 13.04 $\pm$ 0.05 \\
        PCE (P=1) & \textbf{3.91 $\pm$ 0.01} & 3.91 $\pm$ 0.01 & \textbf{7.22 $\pm$ 0.01} & 7.23 $\pm$ 0.01 & 11.57 $\pm$ 0.06 & 12.49 $\pm$ 0.15 \\
        PCE (P=50) & 3.90 $\pm$ 0.01 & 3.90 $\pm$ 0.01 & 7.13 $\pm$ 0.03 & 7.15 $\pm$ 0.03 & 11.15 $\pm$ 0.06 & 11.86 $\pm$ 0.10 \\
        JointVarPost (P=1) & \textbf{3.70 $\pm$ 0.16} & 3.70 $\pm$ 0.16 & 6.10 $\pm$ 0.25 & 6.10 $\pm$ 0.25 & 11.76 $\pm$ 0.08 & 12.99 $\pm$ 0.14 \\
        JointVarPost (P=50) & \textbf{3.95 $\pm$ 0.01} & 3.95 $\pm$ 0.01 & \textbf{7.20 $\pm$ 0.01} & 7.22 $\pm$ 0.01 & \textbf{11.99 $\pm$ 0.03} & 13.37 $\pm$ 0.15 \\
        VarMarg (P=1) & 3.63 $\pm$ 0.02 & 3.64 $\pm$ 0.02 & 5.50 $\pm$ 0.02 & 5.50 $\pm$ 0.02 & 11.44 $\pm$ 0.04 & 12.33 $\pm$ 0.10 \\
        VarMarg (P=50) & 3.63 $\pm$ 0.01 & 3.63 $\pm$ 0.01 & 5.54 $\pm$ 0.00 & 5.55 $\pm$ 0.00 & 11.44 $\pm$ 0.01 & 12.35 $\pm$ 0.04 \\
        VarPost (P=1) & 3.87 $\pm$ 0.01 & 3.87 $\pm$ 0.01 & \textbf{7.21 $\pm$ 0.01} & 7.22 $\pm$ 0.01 & 11.79 $\pm$ 0.02 & 12.92 $\pm$ 0.06 \\
        VarPost (P=50) & \textbf{3.92 $\pm$ 0.02} & 3.92 $\pm$ 0.02 & \textbf{7.19 $\pm$ 0.02} & 7.20 $\pm$ 0.02 & 11.84 $\pm$ 0.02 & 13.32 $\pm$ 0.21 \\
        \iosmc (P=1) & 3.49 $\pm$ 0.04 & 3.49 $\pm$ 0.04 & 5.81 $\pm$ 0.02 & 5.81 $\pm$ 0.02 & 11.31 $\pm$ 0.04 & 12.09 $\pm$ 0.04 \\
        \bottomrule
    \end{tabular}
\end{table*}

The first experiment is a long-standing BED benchmark problem \citep{sheng_maximum_2005, foster_deep_2021} where one places designs in a $d$-dimensional space and makes a noisy observation of the superimposed signals propagating from $K$ unknown sources. We consider the typical $d=2, K=2$ setting and a more challenging $d=3, K=10$ setting. In each case we perform $T=30$ experiments and train a permutation-invariant DAD network~\citep{foster_deep_2021}. We highlight that in each setting, a single pre-trained score network was used to train \emph{all} the score-based policies reported. Similarly, VarMarg and VarPost use a single pre-trained variational approximation while JointVarPost trains entirely from scratch for each policy. 

Results are reported in \Cref{tab:lf_results}. For the $d=2, K=2$ setting we observe that each of the two-stage approaches exhibits marginally improved performance at the $P=5$ protocol, while PCE and the co-trained variational posterior see better performance under $P=1$. This highlights that pre-training of the EIG intractabilities enables cheap training of multiple policies \emph{without} sacrificing the performance of each individual policy. We emphasise that in practical applications, we expect more significant gains by using these $P$ policies to search over architectures and hyperparameters. Here the overall best-performing approach was JointVarPost, closely followed by VarPost and \scorebed.
 
For the $d=3, K=10$ setting, \Cref{tab:lf_k10d3_results} clearly shows that \scorebed outperforms all competing methods at both the $P=1$ and $P=5$ protocols. The parameter space in this task is $30$-dimensional, which we believe particularly suits \scorebed (and challenges nested or variational posterior approaches) due to the behaviour of the score-matching objective in high parameter dimensions, as discussed in \Cref{sec:method}. Interestingly, \scorebed performs marginally better when $P=1$, which indicates that the score network has not yet plateaued at this budget, as corroborated by the EIG vs NLE curves (see below). Surprisingly, we observe that the variational marginal approach does not perform well on either of the location finding settings, despite being theoretically more suited to the high-dimensional setting for the same reasons as \scorebed. 

In \Cref{fig:lf_k10d3_curves} we present a curve of EIG bounds versus NLE budget for each method under the $P=1$ protocol for the higher dimensional setting, which shows \scorebed achieving consistently higher EIG over a range of NLE budgets. We also visualise a comparison between the EIG curves during policy training for \scorebed and PCE, where \scorebed appears to benefit from slightly faster training \emph{per iteration} than PCE, likely due to reducing the variance of the gradient estimates. This translates to significantly faster policy training per NLE since the per-iteration cost of \scorebed is significantly cheaper in NLE given the pre-trained score network allowing us to avoid the need to consider $M$ contrastive samples.

\subsection{Dynamical Systems}

Our second set of experiments is a series of dynamical systems models, introduced as BED benchmarks by \citet{iqbal_nesting_2024}. Each problem consists of a dynamical system whose behaviour is governed by an SDE depending on a set of unknown parameters and controlled by an external force $\xi_t$. The design problem is to control the system via $\xi_t$ in order to learn about the unknown parameters. The models are the stochastic pendulum, cart-pole and double link systems, which increase in complexity based on dimensionality and trajectory dynamics. In all cases, each design decision directly impacts the state of the system at future iterations, creating a non-Markovian structure where early design decisions can dramatically change the regime of later observations. We simulate each system for $T=50$ time steps using an Euler-Maruyama discretisation scheme with step size $\delta=0.05$. We train an adaptive policy using the stochastic LSTM architecture from \citet{iqbal_nesting_2024}.

We note for these problems we experienced some instability in the training of the score network. To mitigate this, we trained three score networks each with a third of the allowed budget. We then split the $P$ policy restarts between the three networks and reported the performance of the best policy as per our established best-of-$P$ protocol. We therefore report $P=3$ and $P=50$ variants only for our approach. We also only present \iosmc in the $P=1$ variant because the wall-clock time for the $P=50$ protocol was prohibitive.

\textbf{Stochastic pendulum.} This system consists of a pendulum of unknown mass and length controlled by a torque $\xi_t \in [-1, 1]$. Results in \Cref{tab:sp_results} show statistically indistinguishable performance between \scorebed and the two variational posterior approaches under the $P=50$ protocol, while PCE performs similarly as well. We found policy training to be very cheap here---only requiring 500 gradient steps to converge---therefore we were able to cheaply train $50$ policy restarts but found only modest improvements in doing so. We expect larger gains could be realised by allowing search over architectures amongst the $P=50$ policies. Both VarMarg and \iosmc performed poorly compared to the other approaches. The lower performance of \iosmc than presented in \citet{iqbal_nesting_2024} throughout the experiments is due to our problem setup being slightly different, with a higher observation noise level to avoid EIG saturation and thus more reliable evaluation. We were still able to replicate their results when run on the same noise setting.

\textbf{Cart-pole.} The cart-pole system consists of a cart of unknown mass and a freely pivoting pole of unknown mass and length. The design $\xi_t \in [-5, 5]$ is the force applied to the cart, and observations consist of the velocity and acceleration of the cart and the angular position and angular velocity of the pole. The results presented in \Cref{tab:cp_results} draw similar conclusions to the stochastic pendulum variant with very minor differences between approaches, except for VarMarg and \iosmc, which performed poorly. We note the joint variational posterior approach also experienced some instability on this task and thus benefitted from selecting the best policy from the $P=50$ protocol.

\textbf{Double link.} The double link system is made up of two links, each of unknown mass and length. The system is controlled by a two-dimensional design with $\xi_t^{(1)} \in [-4, 4]$ being the torque on the first link and $\xi_t^{(2)} \in [-2, 2]$ the torque on the second. The results in \Cref{tab:dp_results} show that the co-trained variational posterior works best on this task, with the variational posterior and \scorebed performing similarly and nearly catching up with the co-trained variant when $P=50$. The best PCE variant here occurs with $P=1$ and lags slightly behind the aforementioned approaches, but is slightly better than VarMarg or \iosmc.

\begin{figure}[t]
    \centering
    \includegraphics[width=0.9\linewidth]{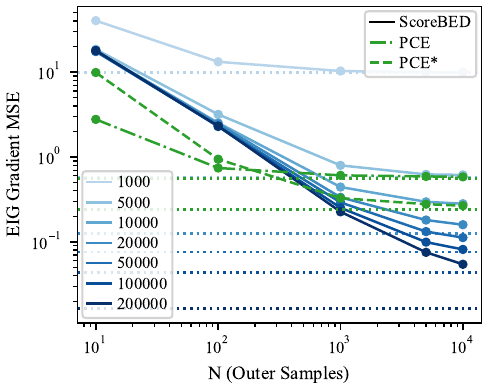}
    \caption{EIG gradient MSE versus N for PCE and \scorebed. Horizontal lines mark squared-bias asymptotes. PCE$^{\ast}$ uses $10\times$ the budget of PCE by increasing the number of inner samples by $10$, as per \Cref{tab:lf_results}.}
    \vspace{-5pt}
    \label{fig:eig_mse_vs_N_lfk10d3}
    \vspace{-5pt}
\end{figure}

\vspace{-3pt}
\subsection{Policy Gradient Bias}
\label{sec:exp_bias_var}
\vspace{-2pt}

\looseness=-1
To investigate the key factors behind the quality of policy training, we explored the error of the \scorebed EIG gradient estimates as a function of $N$ and the number of score training iterations on the high-dimensional source location finding task in \Cref{fig:eig_mse_vs_N_lfk10d3}. At low $N$, we observe the expected $1/N$ convergence rate, which gives way to a bias plateau as $N$ increases.  We derive a bias-variance decomposition in Appendix \ref{app:bias_variance_decomp} which explains this behaviour as variance terms in the decomposition are scaled by $1/N$. Increasing $N$ in the gradient estimator eventually leads to the bias dominating the error, and even at modest $N$---for example we used $N=1024$ in our experiments---the implicit averaging effect of stochastic gradient descent still reveals the bias as the limiting source of error over policy training. Indeed, we overlaid the error curve for PCE and found significantly higher gradient biases, which matches the inferior performance of PCE on this task (whereas we found this bias to be lower for PCE on tasks where it performed better, cf.~\Cref{fig:eig_vs_bias_scatters_lf_sp}). 

\begin{figure}[t]
    \centering
    \includegraphics[width=\linewidth]{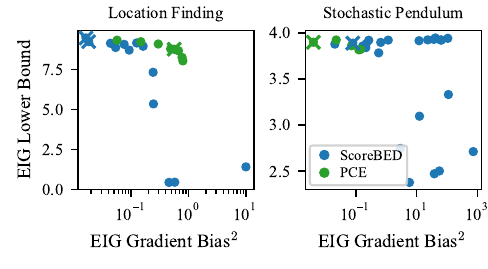}
    \vspace{-2em}
    \caption{EIG lower bound versus EIG gradient squared-bias. \scorebed points are generated from different seeds and training checkpoints. PCE points cover a range of $M$ inner samples. `x' marks the budgets used in our main results.}
    \label{fig:eig_vs_bias_scatters_lf_sp}
    \vspace{-5pt}
\end{figure}

Furthermore, we find a clear correlation between the gradient bias and the downstream EIG achieved across all our experiments, which we illustrate in \Cref{fig:eig_vs_bias_scatters_lf_sp}. Higher gradient biases can throw off the policy training and lead to poor performance of the final policy. This reveals the gradient bias as a key determining factor to the success of any policy training methodology. For \scorebed, bias corresponds to approximation and estimation errors of the score network, which are typically reducible by better training and network capacity, while PCE and nested sampling approaches admit bias through finite sample sizes. This confirms each of these approaches as distinct methods with different properties which may each be advantageous in their own scenarios. 

\vspace{-4pt}
\section{Conclusion}
\vspace{-3pt}

\looseness=-1
We turned the doubly intractable EIG gradient in policy-based BED problems into two singly intractable problems which can be solved sequentially: a score matching problem for the intractable marginal followed by policy-training using the learned score network. This \scorebed approach allows significantly cheaper gradient-based policy training by solving the double intractability upfront, such that we can easily repeatedly train multiple policies on a given problem. This benefits policy training through, for example, allowing architecture and hyperparameter searches and avoiding local optima. Our results illustrate that \scorebed achieves competitive performance with existing approaches that do not allow cheap policy retraining and with newly proposed alternatives that train a fixed amortised posterior network upfront.

\clearpage

\begin{acknowledgements}
    AP is supported by the EPSRC CDT in Modern Statistics and Statistical Machine Learning (EP/S023151/1). GK and TR are supported by EPSRC grant EP/Y037200/1.
\end{acknowledgements}

\bibliography{references}

\newpage

\onecolumn

\title{Bayesian Experimental Design via Score Matching\\(Supplementary Material)}
\maketitle

\appendix

\allowdisplaybreaks[3]

\section*{Introduction}

The Appendix is structured as follows. In Appendix \ref{app:eig_gradient_expression} we derive the gradient expression of \Cref{prop:eig_grad}. In Appendix \ref{app:error_bounds} we prove an error bound on the \scorebed EIG gradient approximation. In Appendix \ref{app:bias_variance_decomp} we provide the bias--variance decomposition discussed in \Cref{sec:exp_bias_var}. In Appendix \ref{app:exp_details} we provide full details of our experimental procedures. In Appendix \ref{app:ablations} we provide additional results, ablation studies and further empirical investigations. 

\section{EIG gradient expression}
\label{app:eig_gradient_expression}

\subsection{Policy-based Total EIG}
\label{app:teig}

Recall the information gain as a measure of the utility~\citep{lindley_measure_1956} of a set of designs $\x$ for which you observe data $\y$, defined as $\IG_\theta(y_{1:T}, \xi_{1:T}) = \mathbb{H}[p(\theta)] - \mathbb{H}[p(\theta\mid y_{1:T}, \xi_{1:T})]$. Since we wish to optimise the next designs prior to observing any data, we consider the expectation of the information gain with respect to the data generating process.

In a policy-based BED setting, we plan to draw designs from an adaptive proposal distribution $\pi_\phi(\xi_t\mid h_{t-1})$ which we assume to be stochastic in general although we note that optimal policies are always deterministic \citep{lindley1972bayesian}. In the stochastic case, $\pi_\phi(\xi_t \mid h_{t-1})$ represents a distribution with density of the same notation defined with respect to the Lebesgue measure. If the policy network is instead deterministic, $\pi_\phi$ no longer has a density with respect to the Lebesgue measure and instead collapses to a degenerate distribution. In the deterministic case our notation is no longer formally correct but an alternative formulation of the likelihood $p(y_t\mid \theta; \xi_t)$ could be considered where $\xi_t$ is no longer a random variable but a parameter of the likelihood. With that distinction, all of the same definitions and results hold. A more in-depth discussion can be found in the Appendix of \citet{ivanova_implicit_2021}.

Given the (stochastic) policy network $\pi_\phi$ and conditional on the underlying parameter $\theta$, the data generating process can be written as $p_\phi(\y, \x\mid\theta) = \prod_{t=1}^T \pi_\phi(\xi_t \mid h_{t-1}) p(y_t\mid\theta, \xi_t, h_{t-1})$. We integrate this over the prior $p(\theta)$ to obtain the joint $p_\phi(\y, \x)$. The posterior distribution is $p(\theta\mid\y, \x) \propto p(\theta)p_\phi(\y, \x\mid\theta)$ which is easily seen to be independent of how the designs $\x$ were proposed, as follows:
\begin{align}
    p(\theta\mid y_{1:T}, \xi_{1:T}) &= \frac{p(\theta)p_\phi(y_{1:T}, \xi_{1:T}\mid\theta)}{\int p(\theta)p_\phi(y_{1:T}, \xi_{1:T}\mid\theta) \mathrm{d}\theta} \\
    &= \frac{p(\theta)\prod_{t=1}^T \pi_\phi(\xi_t\mid h_{t-1}) p(y_t\mid\xi_t, h_{t-1}, \theta)}{\int p(\theta)\prod_{t=1}^T \pi_\phi(\xi_t\mid h_{t-1}) p(y_t\mid\xi_t, h_{t-1}, \theta) \mathrm{d}{\theta}} \label{eqn:policy_dependent_posterior_representation}\\
    &= \frac{p(\theta)\prod_{t=1}^T p(y_t\mid\xi_t, h_{t-1}, \theta)}{\int p(\theta)\prod_{t=1}^T p(y_t\mid\xi_t, h_{t-1}, \theta) \mathrm{d}{\theta}} \label{eqn:policy_independent_posterior}
\end{align}
where the design proposals cancelled. This assumes that observations do not depend on the policy in any way beyond the realised designs, that is we have $p(y_t|\xi_t, h_{t-1}, \theta)$ rather than $p(y_t|\xi_t, h_{t-1}, \theta, \phi)$. The ignorability of the observation mechanism in posterior inference and the conditions required have previously been studied in related literature on missing data, see \citet{rubin1976inference}.

We also define the marginal distribution of the data given a set of designs, $p(\y\mid\x) = \int p(\theta) \prod_{t=1}^T p(y_t\mid\xi_t, h_{t-1}, \theta) \mathrm{d} \theta = \prod_{t=1}^Tp(y_t\mid\xi_t, h_{t-1})$ where $p(y_t\mid\xi_t, h_{t-1}) = \int p(\theta\mid h_{t-1})p(y_t\mid\xi_t, h_{t-1}, \theta)\mathrm{d}\theta$. This decomposition is axiomatically true by stacking nested conditional densities but we also show below that $p(y_t\mid \xi_t, h_{t-1})$ has a particular form:
\begin{align}
    \int p(\theta) &\prod_{t=1}^T p(y_t\mid \xi_t, h_{t-1}, \theta) \mathrm{d} \theta \\
    &= \int p(\theta) p(y_1\mid \xi_1, h_0, \theta) \prod_{t=2}^T p(y_t\mid \xi_t, h_{t-1}, \theta) \mathrm{d}\theta \\
    &= \underbrace{\int p(\theta) p(y_1\mid \xi_1, h_0, \theta) \mathrm{d}\theta}_{p(y_1\mid \xi_1, h_0)} \int \frac{p(\theta) p(y_1\mid \xi_1, h_0, \theta)}{\int p(\theta) p(y_1\mid \xi_1, h_0, \theta) \mathrm{d}\theta} \prod_{t=2}^T p(y_t\mid \xi_t, h_{t-1}, \theta) \mathrm{d}\theta \\
    &= p(y_1\mid \xi_1, h_0) \int p(\theta\mid h_1)\prod_{t=2}^T p(y_t\mid \xi_t, h_{t-1}, \theta) \mathrm{d}\theta \\
    &= p(y_1\mid \xi_1, h_0) \underbrace{\int p(\theta\mid h_1) p(y_2\mid \xi_2, h_1, \theta) \mathrm{d}\theta}_{p(y_2\mid \xi_2, h_1)} \int \frac{p(\theta\mid h_1) p(y_2\mid \xi_2, h_1, \theta)}{\int p(\theta\mid h_1) p(y_2\mid \xi_2, h_1, \theta) \mathrm{d}\theta} \prod_{t=3}^T p(y_t\mid \xi_t, h_{t-1}, \theta) \mathrm{d}\theta \\
    &= \cdots \\
    &= \prod_{t=1}^T p(y_t\mid \xi_t, h_{t-1})
\end{align}
where $p(y_t\mid \xi_t, h_{t-1}) = \int p(\theta\mid h_{t-1})p(y_t\mid \xi_t, h_{t-1}, \theta)\mathrm{d}\theta$ and the intermediate steps were completed by induction. 

With the above distributions carefully defined, we consider the expected information gain as $\mathcal{I}_T(\phi) = \E_{p_\phi(y_{1:T}, \xi_{1:T})}\{\IG_\theta(y_{1:T}, \xi_{1:T})\}$. Note that the information gain itself does not depend on the policy network, but only on the specific set of designs that are seen. The dependency on the policy network is only introduced in the outer expectation of the EIG.

\subsection{Derivation of EIG gradient}
\label{app:eig_grad_expression}

We restate our theorem for the gradient of the EIG in full notation below. 

\begin{theorem}[Gradient of EIG]
\label{prop:eig_grad_app}
    The following expression for the gradient of the EIG holds:
    \begin{equation}
    \label{eqn_app:adaptive_design_grad_prop}
        \begin{split}
            \nabla_\phi \mathcal{I}_T(\phi) &= \mathbb{E}_{p(\theta)q(\delta_{1:T}, \epsilon_{1:T})}\Bigg\{\frac{\mathrm{d}}{\mathrm{d}\phi} \log p(y_{1:T}|\xi_{1:T}, \theta) - \frac{\partial}{\partial y_{1:T}} \log p(y_{1:T}|\xi_{1:T}) \frac{\partial y_{1:T}}{\partial \phi} - \frac{\partial}{\partial \xi_{1:T}} \log p(y_{1:T}|\xi_{1:T}) \frac{\partial \xi_{1:T}}{\partial \phi}\Bigg\}
        \end{split}
    \end{equation}
    where $y_{1:T} = \{g_t^1(\delta_{1:t}, \epsilon_{1:t}, \theta, \phi)\}_{t=1}^T, \; \xi_{1:T} = \{g_{t}^2(\delta_{1:t-1}, \epsilon_{1:t}, \theta, \phi)\}_{t=1}^T \sim p_\phi(\cdot, \cdot \mid \theta)$ for $\delta_{1:T}, \epsilon_{1:T} \sim q(\cdot, \cdot)$ are obtained by reparameterisation.
    In the specific case of static (non-adaptive) designs, the expression simplifies to:
    \begin{equation}
        \begin{split}
            \nabla_{\x} \mathcal{I}_T(\x) = \mathbb{E}_{p(\theta)q(\delta_{1:T})}\bigg\{\frac{\mathrm{d}}{\mathrm{d} \x} & \log p(y_{1:T}|\theta; \x) - \frac{\partial}{\partial y_{1:T}} \log p(y_{1:T}; \x)  \frac{\partial y_{1:T}}{\partial \x}\bigg\}
        \end{split}
    \end{equation}
    where again we use the reparameterisation trick such that $y_{1:T} = \{g_{t}(\delta_{1:t}, \theta, \x)\}_{t=1}^T \sim p(\cdot; \x)$ for $\delta_{1:T} \sim q(\cdot)$.
\end{theorem}
\begin{proof}
To derive this expression we start by considering the mutual information form of the EIG as follows:
\begin{align}
    \mathcal{I}_T(\phi) &= \E_{p_\phi(y_{1:T}, \xi_{1:T})}\{\IG_\theta(y_{1:T}, \xi_{1:T})\} \\
    &= \E_{p_\phi(y_{1:T}, \xi_{1:T})}\{\E_{p(\theta\mid y_{1:T}, \xi_{1:T})}\log p(\theta\mid y_{1:T}, \xi_{1:T}) - \E_{p(\theta)} \log p(\theta) \} \\
    &= \E_{p(\theta) p_\phi(y_{1:T}, \xi_{1:T}\mid \theta)}\{\log p(\theta\mid  y_{1:T}, \xi_{1:T})\} + C \\
    &= \E_{p(\theta) p_\phi(y_{1:T}, \xi_{1:T}\mid \theta)}\{\log p(y_{1:T}\mid  \xi_{1:T}, \theta) - \log p(y_{1:T}\mid \xi_{1:T})\} +C \label{eqn:grad_eig1}
\end{align}

where $C$ denotes terms constant in $\phi$. Firstly, we consider the full gradient expression in the general case of stochastic proposals, given in \cref{eqn:adaptive_design_grad_prop}. This expression follows by reparameterising, differentiating under the expectation, and applying the chain rule. Explicitly, we define the reparameterisation recursively by $y_t = g_t^1(\delta_t, \xi_t, \theta, h_{t-1})$ (recall $y_t \sim p(\cdot\mid  \xi_t, \theta, h_{t-1})$) and $\xi_t = g_t^2(\epsilon_t, h_{t-1}, \phi)$ (recall $\xi_t \sim \pi_\phi(\cdot\mid h_{t-1})$). Rolling out the recursive reparameterisation we obtain\footnote{We use shorthand notation $y_{1:T} = g_{1:T}^1(\delta_{1:T}, \epsilon_{1:T}, \theta, \phi)$ to mean $y_{1:T} = \{g_t^1(\delta_{1:t}, \epsilon_{1:t}, \theta, \phi)\}_{t=1}^T$.} $y_{1:T} = g_{1:T}^1(\delta_{1:T}, \epsilon_{1:T}, \theta, \phi), \; \xi_{1:T} = g_{1:T}^2(\delta_{1:T-1}, \epsilon_{1:T}, \theta, \phi) \sim p_\phi(\cdot, \cdot\mid \theta)$ for $\epsilon_{1:T}, \delta_{1:T}\sim q(\cdot, \cdot)$, as stated in the theorem. Then the gradient of the objective can be written as:
\begin{align}
    \nabla_\phi \mathcal{I}_T(\phi) 
    &= \E_{p(\theta) q(\epsilon_{1:T}, \delta_{1:T})}\Big\{\frac{\mathrm{d}}{\mathrm{d} \phi} \log p(g_{1:T}^1(\dots)\mid  g^2_{1:T}(\dots), \theta) - \frac{\mathrm{d}}{\mathrm{d} \phi} \log p(g^1_{1:T}(\dots)\mid g^2_{1:T}(\dots))\Big\} \label{eqn:standard_nmc_reparam} \\ 
    &= \E_{p(\theta) q(\epsilon_{1:T}, \delta_{1:T})}\Bigg\{\frac{\mathrm{d}}{\mathrm{d} \phi} \log p(g^1_{1:T}(\dots)\mid  g^2_{1:T}(\dots), \theta)  \notag\\ &  \qquad \qquad \qquad \qquad \qquad \qquad -\frac{\partial}{\partial y_{1:T}} \log p(y_{1:T}\mid \xi_{1:T})\Bigg| _{\stackrel{y_{1:T} = g_{1:T}^1(\delta_{1:T}, \epsilon_{1:T}, \theta, \phi)}{\xi_{1:T} = g^2_{1:T}(\delta_{1:T-1}, \epsilon_{1:T}, \theta, \phi)}}\frac{\partial g_{1:T}^1(\delta_{1:T}, \epsilon_{1:T}, \theta, \phi)}{\partial \phi} \notag\\ 
    & \qquad \qquad \qquad \qquad \qquad \qquad -\frac{\partial}{\partial \xi_{1:T}} \log p(y_{1:T}\mid \xi_{1:T})\Bigg| _{\stackrel{y_{1:T} = g_{1:T}^1(\delta_{1:T}, \epsilon_{1:T}, \theta, \phi)}{\xi_{1:T} = g^2_{1:T}(\delta_{1:T-1}, \epsilon_{1:T}, \theta, \phi)}}\frac{\partial g_{1:T}^2(\delta_{1:T-1}, \epsilon_{1:T}, \theta, \phi)}{\partial \phi}\Bigg\}.
\end{align}
This is the main gradient expression of the theorem as required. 

As stated in the beginning, when the designs are static, they are uniquely determined up front and have no dependence on the data or unknown parameters of the model, and so we can simplify the above expression. Starting again from \cref{eqn:grad_eig1}:
\begin{align}
    \E_{p(\theta) p_\phi(y_{1:T}, \xi_{1:T}\mid \theta)}&\{\log p(y_{1:T}\mid  \xi_{1:T}, \theta) - \log p(y_{1:T}\mid \xi_{1:T})\} +C \\
    &= \E_{p(\theta)p(y_{1:T}\mid \theta; \x)}\{\log p(y_{1:T}\mid \theta; \x) - \log p(y_{1:T}; \x)\} + C    
\end{align}
where we have rewritten the marginal distribution of the data given the designs as $p(y_{1:T}\mid \xi_{1:T}) = p(y_{1:T}; \x)$ in the static case (to emphasise the deterministic nature of the designs). This is now easy to reparameterise as $y_t = g_t(\delta_t, \theta, h_{t-1})$ which rolls out to $y_{1:T} = \{g_{t}(\delta_{1:t}, \theta, \xi_{1:t})\}_{t=1}^T$ with $\delta_{1:T} \sim q(\cdot)$. Now differentiating under the expectation and applying the chain rule we obtain:
\begin{align}
    \nabla_{\x} \mathcal{I}_T(\x) = \E_{p(\theta)q(\delta_{1:T})}\bigg\{\frac{\mathrm{d}}{\mathrm{d} \x} \log p(&g_{1:T}(\delta_{1:T}, \theta, \x)\mid \theta; \x)  \notag \\ & -\frac{\partial}{\partial y_{1:T}} \log p(y_{1:T}; \x)\Big|_{y_{1:T} = g_{1:T}(\delta_{1:T}, \theta, \x)}  \frac{\partial g_{1:T}(\delta_{1:T}, \theta, \x)}{\partial \x}  \notag \\
    & -\frac{\partial}{\partial \x} \log p(y_{1:T}; \x)\Big|_{y_{1:T} = g_{1:T}(\delta_{1:T}, \theta, \x)} \bigg\}.
\end{align}
Now we show that the final term is zero in expectation. To do so we undo the original reparameterisation by writing $\delta_{1:T} = h_{1:T}(y_{1:T}, \theta, \x) \sim q(\cdot)$ for $y_{1:T} \sim p(\cdot; \x)$, where $h$ is the inverse of $g$, i.e. $y_{1:T} = g_{1:T}(h_{1:T}(y_{1:T}, \theta, \x), \theta, \x)$. Then we have:
\begin{align}
    \E_{p(\theta)q(\delta_{1:T})}&\bigg\{\frac{\partial}{\partial \x} \log p(y_{1:T}; \x)\Big|_{y_{1:T} = g_{1:T}(\delta_{1:T}, \theta, \x)}\bigg\} \\
    &= \E_{p(\theta)p(y_{1:T}\mid \theta; \x)}\bigg\{\frac{\partial}{\partial \x} \log p(y_{1:T}; \x)\Big|_{y_{1:T} = g_{1:T}(h_{1:T}(y_{1:T}, \theta, \x), \theta, \x)}\bigg\} \\
    &= \E_{p(\theta)p(y_{1:T}\mid \theta; \x)}\bigg\{\frac{\partial}{\partial \x} \log p(y_{1:T}; \x)\bigg\} \\
    &= \iint \frac{\partial}{\partial \x} \log p(y_{1:T}; \x) p(\theta)p(y_{1:T}\mid \theta; \x) \mathrm{d}y_{1:T} \mathrm{d}\theta \\
    &= \iint \frac{\frac{\partial}{\partial \x} p(y_{1:T}; \x)}{p(y_{1:T}; \x)} p(y_{1:T}; \x) p(\theta\mid y_{1:T}) \mathrm{d}y_{1:T} \mathrm{d}\theta \\
    &= \iint \frac{\partial}{\partial \x} p(y_{1:T}; \x) p(\theta\mid y_{1:T}) \mathrm{d}y_{1:T} \mathrm{d}\theta \\
    &= \int \frac{\partial}{\partial \x} p(y_{1:T}; \x) \int p(\theta\mid y_{1:T}) \mathrm{d}\theta \mathrm{d}y_{1:T} \\
    &= \frac{\partial}{\partial \x} \int p(y_{1:T}; \x) \mathrm{d}y_{1:T} = 0.
\end{align}
So we have removed one intractable term from the gradient expression in the static design case, providing the second part of the result from the theorem. 
\end{proof}

One might be tempted to try a similar trick to eliminate the last term in the general case. 
However, when using policies it turns out that this term is no longer zero in expectation.
At a high level, this is because it is necessary to reparameterise the designs, which introduces a dependence on $\theta$ which cannot be integrated away as in the penultimate line of the proof above. For completeness, we show this below. Starting with the last term and applying the inverse reparameterisation trick, analogous to the deterministic case with $\delta_{1:T} = h_{1:T}^1(y_{1:T}, \xi_{1:T}, \theta, \phi), \epsilon_{1:T} = h_{1:T}^2(y_{1:T-1}, \xi_{1:T}, \theta, \phi)$, we have:
\begin{align}
    \E&_{p(\theta)q(\delta_{1:T}, \epsilon_{1:T})}\bigg\{\frac{\partial}{\partial \xi_{1:T}} \log p(y_{1:T}\mid \xi_{1:T})\Bigg| _{\stackrel{y_{1:T} = g_{1:T}^1(\delta_{1:T}, \epsilon_{1:T}, \theta, \phi)}{\xi_{1:T} = g_{1:T}^2(\delta_{1:T-1}, \epsilon_{1:T}, \theta, \phi)}}\frac{\partial g_{1:T}^2(\delta_{1:T-1}, \epsilon_{1:T}, \theta, \phi)}{\partial \phi}\bigg\} \\
    &= \E_{p(\theta)p_\phi(\y, \x\mid \theta)}\bigg\{\frac{\partial}{\partial \x} \log p(\y\mid \x) \times\notag \\ & \qquad \qquad \qquad \qquad \qquad \qquad\underbrace{\frac{\partial g_{1:T}^2(h_{1:T}^1(y_{1:T-1}, \x, \theta, \phi), h_{1:T}^2(y_{1:T-1}, \xi_{1:T-1}, \theta, \phi), \theta, \phi)}{\partial \phi}}_{J_\phi(y_{1:T-1}, \xi_{1:T}, \theta, \phi)}\bigg\} \\
    &= \E_{p(\theta)p_\phi(\y, \x\mid \theta)}\bigg\{\frac{\partial}{\partial \x} \log \prod_{t=1}^Tp(y_t\mid \xi_t, h_{t-1}) J_\phi(y_{1:T-1}, \xi_{1:T}, \theta, \phi)\bigg\} \\
    &= \E_{p(\theta)p_\phi(\y, \x\mid \theta)}\bigg\{\sum_{t=1}^{T-1}\frac{\partial}{\partial \x} \log p(y_t\mid \xi_t, h_{t-1}) J_\phi(y_{1:T-1}, \xi_{1:T}, \theta, \phi) + \notag \\
    & \qquad \qquad \qquad \qquad \qquad \qquad \qquad \qquad \frac{\partial}{\partial \x} \log p(y_T\mid \xi_T, h_{T-1}) J_\phi(y_{1:T-1}, \xi_{1:T}, \theta, \phi) \bigg\}
\end{align}
Looking at the last term:
\begin{align}
    \E&_{p(\theta)p_\phi(\y, \x\mid \theta)}\bigg\{\frac{\partial}{\partial \x} \log p(y_T\mid \xi_T, h_{T-1}) J_\phi(y_{1:T-1}, \xi_{1:T}, \theta, \phi) \bigg\} \\
    &= \iiint \frac{\frac{\partial}{\partial \x} p(y_T\mid \xi_T, h_{T-1})}{p(y_T\mid \xi_T, h_{T-1})} J_\phi(y_{1:T-1}, \x, \theta, \phi) p(\theta\mid \y, \x) \notag \\ & \qquad \qquad \prod_{t=1}^Tp(y_t\mid \xi_t, h_{t-1}) \pi_\phi (\xi_t\mid h_{t-1}) \rmd{\y}\rmd{\x}\rmd{\theta} \\
    &= \iiint \frac{\partial}{\partial \x} p(y_T\mid \xi_T, h_{T-1}) J_\phi(y_{1:T-1}, \x, \theta, \phi) p(\theta\mid \y, \x) \pi_\phi(\xi_{T}\mid h_{T-1}) \notag \\
    & \qquad  \qquad \prod_{t=1}^{T-1}p(y_t\mid \xi_t, h_{t-1}) \pi_\phi (\xi_t\mid h_{t-1}) \rmd{\y}\rmd{\x}\rmd{\theta}. 
\end{align}
The problem now is that we cannot eliminate the dependence on $y_T$ from the posterior by integrating over $\theta$ since the Jacobian term, which arose from the reparameterisation of the designs, depends on $\theta$.

\subsection{Constant Additive Noise Models}

Our gradient expression can be further simplified when the model is known to have constant additive noise, i.e. $y = f(\xi, \theta) + \varepsilon$ where $\varepsilon$ is a realisation of an independent noise variable. In this case, $p(y \mid \theta, \xi) = p_\varepsilon(y - f(\xi, \theta))$ so:
\begin{align*}
    \mathbb{H}[p(y \mid \theta, \xi)] &= \mathbb{E}_{p(y\mid \theta, \xi)} \left\{\log p(y \mid \theta, \xi)\right\} \\ 
    &= \int \log p\left(y \mid \theta, \xi\right) p\left(y \mid \theta, \xi\right) \rmd y \\
    &= \int \log p_\varepsilon\left(y - f(\xi, \theta)\right) p_\varepsilon\left(y - f(\xi, \theta)\right) \rmd y \\
    &= \int \log p_\varepsilon(z) p_\varepsilon(z) \mathrm{d} z \\
    &= \mathbb{H}[p_\varepsilon(\varepsilon)] \equiv \text{constant.}
\end{align*}

Therefore, in such cases the first term of our EIG gradient expression (\cref{eqn:adaptive_design_grad_prop}) is also identically zero and the gradient of the EIG only depends on the intractable score terms which we propose to approximate.

\section{Error Bounds}
\label{app:error_bounds}

In this section we show that the accuracy of our score-based EIG gradient estimator is directly controlled by the accuracy of the score approximation, in the sense that the $\ell_2$ error in the gradient estimator is bounded above by a constant times the mean score error. The constant depends only on the Lipschitz properties of the map from policy parameters to rolled-out trajectories $(\y, \x)$.

\paragraph{Set-up.}
Let $z_\phi := (\y, \x) \in \mathbb{R}^{T(d_y + d_\xi)}$ denote the rolled-out trajectory under the policy with parameters $\phi$, viewed as a function of the reparameterisation variables $(\delta_{1:T}, \epsilon_{1:T})$, the latent $\theta$, and $\phi$:
\begin{equation}
    z_\phi(\delta_{1:T}, \epsilon_{1:T}, \theta) = \big(g^1_{1:T}(\delta_{1:T}, \epsilon_{1:T}, \theta, \phi),\, g^2_{1:T}(\delta_{1:T-1}, \epsilon_{1:T}, \theta, \phi)\big).
\end{equation}
Let $s_\psi : \mathbb{R}^{T(d_y + d_\xi)} \to \mathbb{R}^{T(d_y + d_\xi)}$ be a score approximation and define the corresponding score-based gradient estimator by replacing the marginal score in \Cref{eqn_app:adaptive_design_grad_prop} with $s_\psi$:
\begin{equation}
\label{eqn:app_score_grad_estimator}
    \widehat{G}_\psi(\phi) = \E_{p(\theta) q(\delta_{1:T}, \epsilon_{1:T})}\!\left[\frac{\mathrm{d}}{\mathrm{d}\phi}\log p(\y\mid\x, \theta) - s_\psi(\y, \x)\,\frac{\partial z_\phi}{\partial \phi}\right].
\end{equation}

\begin{theorem}
\label{prop:error_bound}
    Fix a tolerance $\eta > 0$. Assume that, for all $(\delta_{1:T}, \epsilon_{1:T}, \theta)$ in the support of $p(\theta) q(\delta_{1:T}, \epsilon_{1:T})$, the reparameterisation map $\phi \mapsto z_\phi(\delta_{1:T}, \epsilon_{1:T}, \theta)$ is uniformly $L$-Lipschitz, i.e.\ there exists $L > 0$ with
    \begin{equation}
        \norm{z_\phi(\delta_{1:T}, \epsilon_{1:T}, \theta) - z_{\phi'}(\delta_{1:T}, \epsilon_{1:T}, \theta)}_2 \le L \norm{\phi - \phi'}_2 \quad \text{for all } \phi, \phi'.
    \end{equation}
    Suppose further that the score approximation satisfies the mean error bound
    \begin{equation}
    \label{eqn:app_score_error_bound}
        \E_{p,q}\!\left[\norm{s_\psi(\y, \x) - \nabla_{(\y, \x)}\log p(\y\mid\x)}_2\right] \le \eta / L.
    \end{equation}
    Then the $\ell_2$ error of the score-based gradient estimator is bounded by $\eta$ for every $\phi$:
    \begin{equation}
        \norm{\nabla_\phi \mathcal{I}_T(\phi) - \widehat{G}_\psi(\phi)}_2 \le \eta.
    \end{equation}
\end{theorem}

\begin{proof}
    Subtracting the approximate gradient $\widehat{G}_\psi$ in \Cref{eqn:app_score_grad_estimator} from the EIG gradient expression \Cref{eqn_app:adaptive_design_grad_prop} gives
    \begin{equation}
        \nabla_\phi \mathcal{I}_T(\phi) - \widehat{G}_\psi(\phi) = \E_{p,q}\!\left[\big(s_\psi(\y, \x) - \nabla_{(\y, \x)}\log p(\y\mid\x)\big)\,\frac{\partial z_\phi}{\partial \phi}\right].
    \end{equation}
    Taking $\ell_2$ norms, by Jensen's inequality and submultiplicativity of the operator norm,
    \begin{align}
        \norm{\nabla_\phi \mathcal{I}_T(\phi) - \widehat{G}_\psi(\phi)}_2
        &\le \E_{p,q}\!\left[\norm{\big(s_\psi(\y, \x) - \nabla_{(\y, \x)}\log p(\y\mid\x)\big)\,\frac{\partial z_\phi}{\partial \phi}}_2\right] \\
        &\le \E_{p,q}\!\left[\norm{s_\psi(\y, \x) - \nabla_{(\y, \x)}\log p(\y\mid\x)}_2 \cdot \norm{\frac{\partial z_\phi}{\partial \phi}}_{\mathrm{op}}\right].
    \end{align}
    Since $\phi \mapsto z_\phi$ is $L$-Lipschitz almost surely, Rademacher's theorem \citep{evans2025measure} gives $\norm{\partial z_\phi / \partial \phi}_{\mathrm{op}} \le L$ almost surely. Pulling this constant out of the expectation gives
    \begin{equation}
        \norm{\nabla_\phi \mathcal{I}_T(\phi) - \widehat{G}_\psi(\phi)}_2 \le L \cdot \E_{p,q}\!\left[\norm{s_\psi(\y, \x) - \nabla_{(\y, \x)}\log p(\y\mid\x)}_2\right] \le \eta,
    \end{equation}
    where the final inequality uses the assumed score-error bound \Cref{eqn:app_score_error_bound}.
\end{proof}

\Cref{prop:error_bound} confirms the intuition motivating our approach: the EIG-gradient error is directly controlled by the score-matching error. Reducing the mean score error by a factor $k$ reduces the gradient error by the same factor (up to the Lipschitz constant $L$ of the reparameterisation, which is independent of the score network). Since score-matching training has access to unlimited samples from the model, achieving small mean score error is in principle only a question of network capacity and training budget.

\paragraph{Remark (component-separated bound).}
A slightly sharper bound is available when the $y$- and $\xi$-reparameterisation maps have different Lipschitz constants. If $\phi \mapsto y_{1:T}$ and $\phi \mapsto \xi_{1:T}$ are uniformly Lipschitz with constants $L_y$ and $L_\xi$ respectively, then by the same argument applied to each score component separately,
\begin{equation}
    \label{eqn:component_separate_err_bounds}
    \norm{\nabla_\phi \mathcal{I}_T(\phi) - \widehat{G}_\psi(\phi)}_2 \le L_y \cdot \E_{p,q}\norm{e_y}_2 + L_\xi \cdot \E_{p,q}\norm{e_\xi}_2,
\end{equation}
where $e_y$ and $e_\xi$ are the $y$- and $\xi$-component errors of the score network. This refinement is useful when one of the two reparameterisation maps is significantly more regular than the other.

\section{Bias-Variance Decompositions}
\label{app:bias_variance_decomp}

We now decompose errors in the score network and resulting EIG gradient approximation into interpretable bias and variance components, revealing the behaviour of our EIG gradient approximation.

\paragraph{Notation.} For notational simplicity we drop the $t$ component of our data and design variables and we write $s(\xi,y)$ for the score network, $\mu(\xi,y) = \nabla_{y,\xi}\log p(y|\xi)$ for the true marginal score. We define $\Sigma_\nu(\xi,y) := \Cov_{\theta|\xi,y}(\nabla \log p(y|\xi,\theta))$ as the posterior covariance of the conditional score.

\subsection{Score Error}

Firstly we consider a bias-variance decomposition of the score matching loss. For the MSM loss (\Cref{eqn:msm_loss}), adding and subtracting $\mu$ inside the loss gives
\begin{align}
    \mathcal{L} &= \E_{\xi,y}\E_{\theta|\xi,y}\|s(\xi,y) - \nabla \log p(y|\xi,\theta)\|^2 \\
    &= \E_{\xi,y}\|s(\xi,y) - \mu(\xi,y)\|^2 + \E_{\xi,y}\Tr\Sigma_\nu(\xi,y),
\end{align}
where we used $\E_{\theta|\xi,y}[\nabla \log p(y|\xi,\theta)] = \mu$ to eliminate the cross term. In this expression the second term isolates the irreducible target variance $\E_{\xi,y}\Tr\Sigma_\nu(\xi,y)$ from the learnable score network error in the first term. Applying the standard bias--variance identity $\E\|e\|^2 = \|\E e\|^2 + \Tr\Cov(e)$ to the first term with error $e = s - \mu$ gives
\begin{equation}
    \mathcal{L} = \underbrace{\|B\|^2}_{\text{bias}} + \underbrace{V_e}_{\text{error variance}} + \underbrace{V_\nu}_{\text{target noise}},
\end{equation}
with $B := \E_{\xi,y}[s-\mu]$, $V_e := \Tr\Cov_{\xi,y}(s-\mu)$ and $V_\nu := \E_{\xi,y}\Tr\Sigma_\nu(\xi, y)$. Empirically, in our experiments we found the dominant term in this decomposition to be the error variance $V_e$ while the irreducible target variance $V_\nu$ was insignificant in the overall loss.

\subsection{EIG Gradient Error}

We now consider analogous decompositions for the EIG gradient error of \scorebed. For ease of notation we consider gradients with respect to a single static design $\xx$, with the policy gradient versions coming from a single pushback through the policy network Jacobian. 

Recalling the form of the \scorebed EIG gradient approximation in \Cref{eqn:adaptive_design_grad_prop}, the estimate is formed by an expectation over $y \sim p(\cdot | \xx)$ of the per-sample integrand which we denote by:
\begin{equation}
    g_s(\xx,y) = s^\xi(\xx,y) + J(\xx,y)^\top s^y(\xx,y),
\end{equation}
where $s^\xi, s^y$ are the design- and data-score outputs of the network and $J(\xx,y) = \partial y/\partial\xx$ is the reparameterisation Jacobian. We write $g_\mu$ for the same integrand built with the true score $\mu$. The \scorebed gradient approximation is then realised by an $N$-sample Monte Carlo estimate of the expectation which we denote
\begin{equation}
    \hat{\mathcal{G}}_N(\xx) := \frac1N\sum_{i=1}^N g_s(\xx,y_i),
\end{equation}
while the true gradient is denoted $\mathcal{G}(\xx) = \E_{y|\xx}[g_\mu(\xx,y)]$. 

We now apply the standard bias--variance decomposition to the MSE of the \scorebed EIG gradient approximation. Since the target $\mathcal{G}(\xx)$ is deterministic, the identity $\E\|Z - c\|^2 = \|\E Z - c\|^2 + \Tr\Cov(Z)$ applied to $Z = \hat{\mathcal{G}}_N(\xx)$ gives
\begin{equation}
  \mathrm{MSE}_N(\xx) := \E_{\yn|\xx}\big\|\hat{\mathcal{G}}_N(\xx) - \mathcal{G}(\xx)\big\|^2 = \underbrace{\big\|\E_{\yn|\xx}[\hat{\mathcal{G}}_N(\xx)] - \mathcal{G}(\xx)\big\|^2}_{\text{squared bias}} + \underbrace{\Tr\Cov_{\yn|\xx}(\hat{\mathcal{G}}_N(\xx))}_{\text{variance}}.
\end{equation}
For the bias, linearity of the sample average gives $\E_{\yn|\xx}[\hat{\mathcal{G}}_N(\xx)] = \E_{y|\xx}[g_s]$, so
\begin{equation}
  \E_{\yn|\xx}[\hat{\mathcal{G}}_N(\xx)] - \mathcal{G}(\xx) = \E_{y|\xx}[g_s - g_\mu] =: B(\xx),
\end{equation}
which is independent of $N$. For the variance, the outer samples $y_i \sim p(\cdot|\xx)$ are i.i.d., so the covariance of their average is $1/N$ times the single-sample covariance,
\begin{equation}
  \Tr\Cov_{\yn|\xx}\Big(\tfrac1N\textstyle\sum_{i=1}^N g_s(\xx,y_i)\Big) = \tfrac1N\Tr\Cov_{y|\xx}(g_s) =: \tfrac1N V_s(\xx),
\end{equation}
where $V_s(\xx)$ is the per-sample variance of the score integrand. Combining the two terms,
\begin{equation}
  \boxed{\;\mathrm{MSE}_N(\xx) = \|B(\xx)\|^2 + \tfrac1N V_s(\xx).\;}
\end{equation}

Overall, we observe the variance contributions to the EIG gradient error are scaled down by $\frac{1}{N}$, leaving the \emph{gradient bias as the dominant term} which determines the quality of the policy training.

\paragraph{Relating $B$ to $B(\xx)$.} It remains to show how the bias $B=\E_{\xi,y}[s-\mu]$ in the score matching loss decomposition relates to the EIG gradient bias $B(\xx) = \E_{y|\xx}[g_s - g_\mu]$, and indeed whether reducing $B$ corresponds to a meaningful reduction in $B(\xx)$.

To do so, we observe that $B$ is an \emph{unconditional} mean of the score error over the joint distribution of $(\xi,y)$, whereas $B(\xx)$ is a \emph{conditional} mean of the integrand error at a fixed design. We therefore introduce the design-conditional score bias:
\begin{equation}
    b(\xi) := \E_{y|\xi}[s(\xi,y) - \mu(\xi,y)],
\end{equation}
which recovers $B = \E_\xi[b(\xi)]$ by the tower property. This design-conditional bias offers a conditional version of the bias--variance decomposition of the MSM loss:
\begin{align}
    \E_{\xi,y}\|s(\xi,y) - \mu(\xi,y)\|^2 &= \E_\xi \E_{y|\xi} \|s(\xi, y) - \mu(\xi, y)\|^2 \\
    &= \E_\xi[\|b(\xi)\|^2 + \Tr\Cov_{y|\xi}(s(\xi, y) - \mu(\xi, y))] \\
    &= \underbrace{\E_\xi\|b(\xi)\|^2}_{\text{design-averaged squared bias}} + V_e^{\text{within}},
\end{align}
where $V_e^{\text{within}} = \E_{\xi}[\Tr\Cov_{y|\xi}(s(\xi, y) - \mu(\xi, y))]$, which gives the MSM loss decomposition:
\begin{equation}
    \mathcal{L} = \underbrace{\E_\xi\|b(\xi)\|^2}_{\text{design-averaged squared bias}} + V_e^{\text{within}} + V_\nu.
\end{equation}
This confirms that the loss penalises $\E_\xi\|b(\xi)\|^2$, the mean \emph{squared} conditional bias --- not merely $\|B\|^2 = \|\E_\xi b(\xi)\|^2$. In particular, conditional biases of opposite sign at different designs cancel in $B$ but not in $\E_\xi\|b(\xi)\|^2$.

It remains to connect the conditional score bias $b(\xi)$ to the integrand bias $B(\xx)$. Splitting the integrand error into its design and data components,
\begin{equation}
    g_s - g_\mu = (s^\xi - \mu^\xi) + J(\xx,y)^\top(s^y - \mu^y),
\end{equation}
the design component of $B(\xx)$ is exactly the design block of the conditional bias $b(\xx)$, so the argument above transfers directly: reducing the loss drives down the design-averaged squared design bias. The data component is instead the \emph{Jacobian-weighted} conditional mean $\E_{y|\xx}[J^\top(s^y-\mu^y)]$, which the (unweighted, $J$-agnostic) loss controls only through a Cauchy--Schwarz bound
\begin{equation}
    \big\|\E_{y|\xx}[J^\top(s^y-\mu^y)]\big\|^2 \le \big(\E_{y|\xx}\|J\|_{\mathrm{op}}^2\big)\big(\E_{y|\xx}\|s^y-\mu^y\|^2\big),
    \label{eqn:app_data_error_cauchy_schwarz}
\end{equation}
with a constant set by the reparameterisation geometry. This directly bounds the data-component contribution to $\|B(\xx)\|^2$ by the per-design score MSE $\E_{y|\xx}\|s^y-\mu^y\|^2$ which the loss controls. This is the same mechanism which is identified in our error bound proof of Appendix \ref{app:error_bounds}.

In summary, we confirm through the illustrative bias-variance decompositions presented above that the score matching loss penalises the score network in such a way that reduces the bias in the error of the EIG gradient estimate, either directly in the case of the design score term or indirectly via the Cauchy--Schwarz inequality for the data score term. 

\paragraph{Penalising design-conditional bias.} We briefly highlight two score identities which allow one to directly penalise the design-conditional bias of the score network during training. For $y\in\mathbb{R}^d$, assuming boundary terms vanish, the divergence theorem gives
\begin{equation}
    \E_{y|\xi}[\nabla_y \log p(y|\xi)] = \int \nabla_y p(y|\xi) \,\rmd y = 0,
\end{equation}
and for the design component
\begin{equation}
    \E_{y|\xi}[\nabla_\xi \log p(y|\xi)] = \nabla_\xi \int p(y|\xi) \,\rmd y = 0.
\end{equation}
Thus the true marginal score integrates to zero over $y|\xi$, and at fixed $\xi$ the bias derives from the network alone: $\|b(\xi)\|^2 = \|\E_{y|\xi} s(\xi, y)\|^2$. Given samples $\{y_i\}\sim p(\cdot|\xi)$, one could estimate and penalise this bias as an addition to the MSM loss. In particular, the naive plug-in $\|\frac1M\sum_i s(\xi,y_i)\|^2$ is itself biased upward by $\tfrac1M\Tr\Cov_{y|\xi}(s)$ but an unbiased estimate could be obtained via the U-statistic
\begin{equation}
    U(\xi) = \frac{1}{M(M-1)} \sum_{i \neq j} s(\xi, y_i)^\top s(\xi, y_j).
\end{equation}
This could be included in the score network training objective to provide additional signal on the critical design-conditional bias, while leaving the optimum unchanged.

\section{Experimental Details}
\label{app:exp_details}

Here we detail the experimental set-up used to produce the results in \Cref{sec:experiments}. Our approach is implemented in Jax \citep{bradbury_jax_2018} and all training was performed on a single NVIDIA Hopper H100 GPU. We benchmark against several baselines: PCE, an unbiased Monte Carlo approach MLMC \citep{goda_unbiased_2022}, pre-trained variational marginal and posterior approaches \citep{foster_variational_2019}, a joint variational posterior \citep{foster_unified_2020} trained on the Barber-Agakov lower bound, and finally \iosmc \citep{iqbal_nesting_2024}. We implemented the PCE, MLMC, and all variational baselines ourselves and used the authors' own implementation \citep{iqbal_nesting_2024} of \iosmc. Our results tables report NMC/PCE bounds on EIG performance of the final policies, with error bounds representing standard error with respect to randomness in the selected policies. We also explore the stability of the score network training in Appendix \ref{app:score_training_stability}.

This section is organised as follows. In Appendix \ref{app:bed_tasks} we describe the generative model for each of our experimental tasks. In Appendix \ref{app:policy_nets} we detail the policy networks used and their training parameters. In Appendix \ref{app:method_score} we give details of the score matching approach followed by details of each baseline in Appendix \ref{app:method_baselines}. In Appendix \ref{app:evaluation} we provide details of the PCE/NMC bounds and the sample sizes used for evaluation. Finally, in Appendix \ref{app:nle} we calculate the NLE cost of each method and provide the exact parameters used to obtain the compute-normalised comparisons in our main results. 

Our code is available here: \url{https://github.com/angusphillips/Score_Matching_BED}

\subsection{BED tasks}
\label{app:bed_tasks}

\subsubsection{Source location finding}
\label{app:lf_task}

\paragraph{Model definition.}
For the source location finding task we use the exact set-up of \citet{foster_deep_2021}. The model consists of a Gaussian prior on $K$ hidden sources in $\mathbb{R}^d$, $\theta_k \stackrel{\mathrm{i.i.d.}}{\sim} \mathcal{N}(0, I_d)$ for $k=1, \dots, K$, and a log-normal observation model
\begin{equation}
    \log y \mid \theta, \xi \sim \mathcal{N}\big(\log \mu(\theta, \xi),\, \sigma^2\big),
\end{equation}
where the mean function is a superposition of decaying signals from each source,
\begin{equation}
    \mu(\theta, \xi) = b + \sum_{k=1}^K \frac{\alpha_k}{m + \norm{\theta_k - \xi}^2}.
\end{equation}
We use $b = 10^{-1}$, $m = 10^{-4}$, $\alpha_k = 1$, $\sigma = 0.5$. In the standard location finding task we consider $d = 2$, $K = 2$ and $T = 30$, while the high-dimensional setting considers $d=3$, $K=10$ and $T=30$. We therefore have $y_t \in \mathbb{R}$ and $\xi_t \in \mathbb{R}^2$ or $\xi_t \in \mathbb{R}^3$ respectively. 

\subsubsection{Dynamical systems}
\label{app:dyn_systems_tasks}

The three dynamical-systems tasks follow the experimental set-up of \citet{iqbal_nesting_2024} but, for completeness, we restate the generative models in full here. In each case the latent parameters are physical constants with a log-normal prior $\log \theta \sim \mathcal{N}(0,\, 0.01\, I_{d_\theta})$ and the data $y_t \in \mathbb{R}^{d_y}$ is generated by an Euler--Maruyama discretisation of the corresponding controlled SDE with step size $\delta t = 0.05$,
\begin{equation}
\label{eqn:em_step}
    y_t = y_{t-1} + \delta t \cdot f(y_{t-1}, \xi_t, \theta) + \sqrt{\delta t}\, L\, \epsilon_t,
    \quad \epsilon_t \sim \mathcal{N}(0, I_{d_y}),
\end{equation}
where $f$ is the system-specific drift and $L$ is a diagonal diffusion matrix. We use $T=50$ for all three systems and initialise all systems with $y_0 = 0$. The likelihood $p(y_t \mid y_{t-1}, \xi_t, \theta) = \mathcal{N}(y_t;\, y_{t-1} + \delta t \cdot f(y_{t-1}, \xi_t, \theta),\, \delta t \, L L^\top)$ is Markovian (i.e. \emph{not} i.i.d.\ across $t$).

\paragraph{Stochastic pendulum.}
\label{app:stoch_pend}
The state $y = (q, \dot q) \in \mathbb{R}^2$ contains the angle $q$ and angular velocity $\dot q$ of a damped pendulum driven by a scalar control torque $\xi \in [-1, 1]$. The latent parameters $\theta = (m, l)$ are the mass and length, both with log-normal priors. The drift in \Cref{eqn:em_step} is
\begin{equation}
    f(y, \xi, \theta) = \begin{pmatrix} \dot q \\[2pt] -\dfrac{3\,g}{2\,l}\sin(q) + \dfrac{3\,(\xi - b\,\dot q)}{m\,l^{2}} \end{pmatrix},
\end{equation}
where $g = 9.81$ is gravitational acceleration and $b = 0.1$ is a damping coefficient.\footnote{There is an inconsistency between the equations stated in \citet{iqbal_nesting_2024} and their public implementation: the stated drift omits the factor of $3$ on the input/damping term, while their implementation includes it. To remain consistent with the results in their paper we follow the implementation.} We use the diffusion matrix $L = \mathrm{diag}(0.1, 0.1)$. Designs are constrained to $[-1, 1]$ by clipping, which is done via a tanh output activation on the final stochastic output of the policy network.

\paragraph{Cart-pole.}
\label{app:cart_pole}
The state $y = (s, q, \dot s, \dot q) \in \mathbb{R}^4$ contains the cart position $s$, pole angle $q$ and their velocities. The latent parameters $\theta = (l, m_p, m_c)$ are the pole length, pole mass and cart mass, all with log-normal priors. A scalar control force $\xi \in [-5, 5]$ is applied to the cart. The drift in \Cref{eqn:em_step} is
\begin{align}
    f(y, \xi, \theta) &= \left( \dot{s}, \dot{q}, \ddot{s}, \ddot{q}\right)^T, \\
    \ddot s &= \frac{\xi + m_p \sin(q)\!\left(l \dot q^2 + g \cos(q)\right) - (k_1 s + d_1 \dot s) - (k_2 q + d_2 \dot q)\,\cos(q)/l}{m_c + m_p \sin^2(q)}, \\
    \ddot q &= \frac{-\xi \cos(q) - m_p\, l\, \dot q^2 \cos(q)\sin(q) - (m_c + m_p)\,g\,\sin(q) - (k_1 s + d_1 \dot s)\cos(q) - (k_2 q + d_2 \dot q)\cos^2(q)/l}{l\,(m_c + m_p \sin^2(q))},
\end{align}
with stiffness/damping constants $(k_1, k_2, d_1, d_2) = (0.01, 0.01, 0.01, 0.01)$ and $g = 9.81$. We use the diffusion matrix $L = \mathrm{diag}(0.1, 0.1, 0.1, 0.1)$. Again designs are constrained to $[-5, 5]$ with a tanh output activation on the policy network.

\paragraph{Double link.}
\label{app:double_pend}
The state $y = (q_1, q_2, \dot q_1, \dot q_2) \in \mathbb{R}^4$ contains the angles and angular velocities of the two links. The latent parameters $\theta = (m_1, m_2, l_1, l_2)$ are the masses and lengths of the two links, all with log-normal priors. A $2$-dimensional control torque $\xi \in [-4, 4]\times[-2, 2]$ is applied. The Euler-Lagrange equations give the constrained dynamics
\begin{equation}
    M(q_2; \theta) \begin{pmatrix} \ddot q_1 \\ \ddot q_2 \end{pmatrix} + C(q_2, \dot q_1, \dot q_2; \theta) \begin{pmatrix} \dot q_1 \\ \dot q_2 \end{pmatrix} + \tau(q_1, q_2; \theta) = \xi,
\end{equation}
with mass matrix
\begin{equation}
    M(q_2; \theta) = \begin{pmatrix} (m_1 + m_2) l_1^2 + m_2 l_2^2 + 2 m_2 l_1 l_2 \cos(q_2) & m_2 l_2^2 + m_2 l_1 l_2 \cos(q_2) \\ m_2 l_2^2 + m_2 l_1 l_2 \cos(q_2) & m_2 l_2^2 \end{pmatrix},
\end{equation}
coriolis/centrifugal matrix
\begin{equation}
    C(q_2, \dot q_1, \dot q_2; \theta) = \begin{pmatrix} 0 & -m_2 l_1 l_2 \sin(q_2)\,(2 \dot q_1 + \dot q_2) \\ \tfrac{1}{2} m_2 l_1 l_2 (2 \dot q_1 + \dot q_2) \sin(q_2) & -\tfrac{1}{2} m_2 l_1 l_2 \dot q_1 \sin(q_2) \end{pmatrix},
\end{equation}
and gravity term
\begin{equation}
    \tau(q_1, q_2; \theta) = -g \begin{pmatrix} (m_1 + m_2) l_1 \sin(q_1) + m_2 l_2 \sin(q_1 + q_2) \\ m_2 l_2 \sin(q_1 + q_2) \end{pmatrix}.
\end{equation}
The drift in \Cref{eqn:em_step} is then $f(y, \xi, \theta) = (\dot q_1,\, \dot q_2,\, M^{-1}(\xi - C \dot q - \tau))$. We use the diffusion matrix $L = \mathrm{diag}(0.1, 0.1, 0.1, 0.1)$ and $g = 9.81$. Again designs are constrained with a tanh output activation on the final output of the policy network.

\subsection{Gravimetry}
\label{app:gravimetry_details}

\paragraph{Model definition.}
The gravimetry experiments are inspired by the gravity methods of \citet{Telford_Geldart_Sheriff_1990} for detecting an underground void or well from surface measurements of the local gravitational anomaly. A buried spherical mass anomaly of source strength $\kappa$ at horizontal position $x_0$ and depth $z$ produces, at a surface measurement location $\xi_t$, a vertical gravity perturbation (in $\mu$Gal, with lengths in metres)
\begin{equation}
    g(\xi_t; \theta) = \kappa\, \frac{z}{\big((\xi_t - x_0)^2 + z^2\big)^{3/2}}.
\end{equation}
The unknown parameter is $\theta = (x_0, z, \kappa) \in \mathbb{R} \times \mathbb{R}_{\geq0} \times \mathbb{R}_{\leq0}$. The source strength $\kappa$ (in $\mu\mathrm{Gal}\cdot\mathrm{m}^2$) is negative for a mass deficit such as a void, and relates to the physical radius $R$ and density contrast $\Delta\rho$ of the anomaly through
\begin{equation}
    \kappa = \frac{G\cdot\frac{4}{3}\pi R^3\,\Delta\rho}{10^{-8}},
\end{equation}
i.e.\ the mass-anomaly term $G\,\Delta M$ rescaled from $\mathrm{m\,s^{-2}}$ to $\mu$Gal with $1\ \mu\mathrm{Gal} = 10^{-8}\ \mathrm{m\,s^{-2}}$. We fix $G = 6.674\times10^{-11}\ \mathrm{m^3\,kg^{-1}\,s^{-2}}$ and a density contrast $\Delta\rho = -2200\ \mathrm{kg\,m^{-3}}$ (typical of an air-filled void in rock). Only the product $R^3\Delta\rho$ (equivalently $\kappa$) is identifiable, so we fix $\Delta\rho$ and never infer $R$ and $\Delta\rho$ jointly. The observations are i.i.d.\ additive Gaussian around this nonlinear mean,
\begin{equation}
    y_i \mid \theta, \xi_i \sim \mathcal{N}\!\big(g(\xi_i; \theta),\, \sigma^2\big),
    \qquad
    p(\y \mid \theta, \xi) = \prod_{i=1}^N \mathcal{N}\!\big(y_i;\, g(\xi_i; \theta),\, \sigma^2\big),
\end{equation}
with observation noise scale $\sigma = 1\ \mu$Gal.

The prior on the latent parameters $\theta = (x_0, z, \kappa)$ is
\begin{equation}
    x_0 \sim \mathcal{N}(0,\, \tau^2), \qquad
    \log z \sim \mathcal{N}(\ln 10,\, 0.5^2)\ \ (z \ge 3), \qquad
    R \sim \mathcal{U}(1.5,\, 4)\ \mathrm{m},
\end{equation}
where $\tau = 30$ m and the source strength $\kappa$ is obtained generatively from the sampled radius $R$ through the mapping above with $\Delta\rho$ fixed, giving $\kappa \in [-3940, -210]\ \mu\mathrm{Gal}\cdot\mathrm{m}^2$. This places the horizontal source position within roughly $\pm 60$ m of the origin, the depth at a median of $10$ m (with $95\%$ mass in $[3.7, 27]$ m and a floor at $3$ m), and the source strength at typical void magnitudes.

\subsection{Policy networks}
\label{app:policy_nets}

For each task we use the same policy network architecture across all methods, so the resulting policies differ only through the gradient estimator or alternative method that trains them. In total we use three different types of policy network across the experimental tasks. Firstly, in both settings of location finding we used the DAD policy network \citep{foster_deep_2021}. For each of the dynamical systems models we used a stochastic LSTM with GRU cells as per \citet{iqbal_nesting_2024}. Finally, in the gravimetry experiment we used the iDAD attention network introduced in \citet{ivanova_implicit_2021}. Optimiser settings for the policy networks are also shared across all methods and are summarised in \Cref{tab:policy_training}; these settings are largely borrowed from their original implementations.

\paragraph{DAD policy network (location finding).}
We use the deep adaptive design (DAD) policy of \citet{foster_deep_2021}. Each $(y_t, \xi_t)$ pair is concatenated and passed through a shared encoder MLP with hidden layers $[256]$ and ReLU activations, producing a $16$-dimensional encoding per time step. The history encoding is formed by summing the per-step encodings, which is permutation-invariant by construction, and is passed through a linear emitter to produce the next design $\xi_t \in \mathbb{R}^2$ (we reduce default initialisation of the weights in the final emitter layer from $\mathcal{U}(-\frac{1}{\sqrt{f_\text{in}}}, \frac{1}{\sqrt{f_{\text{in}}}})$ to $\mathcal{U}(-0.05, 0.05)$ to control the dynamic range of designs at initialisation).

\paragraph{Iqbal GRU policy network (dynamical systems).}
For all three dynamical-systems tasks we use the stochastic LSTM policy network of \citet{iqbal_nesting_2024}. Each $(y_t, \xi_t)$ pair is concatenated and passed through a shared encoder MLP with hidden layers $[256, 256]$ and ReLU activations, producing a $64$-dimensional encoding per time step. The sequence of encodings is processed by a $2$-layer GRU of hidden size $64$, with the top-layer hidden state mapped through an emitter MLP with hidden layers $[256, 256]$ to the design dimensionality. Policies in this case are stochastic, following the original implementation by \citet{iqbal_nesting_2024}. We therefore let the emitter output mean $\mu_t$ and we parameterise a per-component learnable standard deviation in log space, so that $\xi_t \sim \mathcal{N}(\mu_t,\, \mathrm{diag}(\sigma^2))$.

\paragraph{Attention policy network (gravimetry).}
For the gravimetry tasks we use a transformer-based policy inspired by the iDAD policy network~\citep{ivanova_implicit_2021}. Each design $\xi_t$ and outcome $y_t$ is first embedded separately by design- and outcome-encoder MLPs, each with hidden layers $[64]$ and GELU activations producing $32$-dimensional embeddings; these are concatenated and passed through a fusion MLP (hidden layers $[64]$, GELU) to give a $32$-dimensional per-step pair representation $r_t$. The set of history representations $\{r_1, \dots, r_t\}$, each augmented with a learned history-type embedding, is prepended with a learned decision token and processed by a single self-attention transformer encoder layer with $4$ heads, model dimension $32$ and a position-wise feed-forward network of hidden size $64$ (using GELU activation, residual connections and layer normalisation throughout). A causal mask ensures that when emitting the design at step $t$ the decision token attends only to itself and the $t$ available history tokens, so the aggregation is invariant to the presentation order of the observations. The design is emitted from the decision-token output using an MLP with hidden layers $[64]$ and GELU activations. As with the DAD policy, we reduce the default initialisation of the final emitter layer weights to $\mathcal{U}(-0.05, 0.05)$ to control the dynamic range of designs at initialisation.

\paragraph{Policy output activations.}
In order to respect any scaling or bound constraints on designs $\xi_t$, we apply output activations to the raw outputs of the policy network which are detailed in \Cref{tab:output_activations}.

\begin{table}[t]
    \centering
    \begin{tabular}{lr}
        \toprule
        Task & output activation \\
        \midrule
        Location finding & $\xi_t = u_t$ \\
        Stochastic pendulum & $\xi_t = \tanh(u_t)$\\
        Cart-pole & $\xi_t = 5 \tanh(u_t)$\\
        Double pendulum & $\xi_t = (4, 2) \odot \tanh(u_t)$\\
        Gravimetry & $\xi_t = 40 u_t$ \\
        \bottomrule
    \end{tabular}
    \caption{Policy output activations by task. $\odot$ represents element-wise multiplication and $\tanh$ is applied elementwise.}
    \label{tab:output_activations}
\end{table}

\paragraph{Optimiser settings.}
Each policy is trained with the Adam optimiser, gradient clipping at $\|\nabla\|_2 \le 1$, and an exponential learning-rate decay schedule. The full hyperparameters are given in \Cref{tab:policy_training}.

\begin{table}[h]
\centering
\begin{tabular}{lccccccc}
\toprule
Task & \# iters & batch size & schedule & init LR & decay rate & decay every & $(\beta_1, \beta_2)$ \\
\midrule
Location finding  & $20{,}000$ & $1024$ & exp decay & $5\times 10^{-5}$ & $0.98$ & $1000$ & $(0.8, 0.998)$ \\
Stochastic pendulum & $500$ & $1024$ & constant & $1\times 10^{-4}$ & N/A & N/A  & $(0.9, 0.999)$ \\
Cart-pole          & $500$ & $1024$ & constant & $1\times 10^{-4}$ & N/A & N/A  & $(0.9, 0.999)$ \\
Double link        & $500$ & $1024$ & constant & $1\times 10^{-4}$ & N/A & N/A  & $(0.9, 0.999)$ \\
Gravimetry        & $10000$ & $1024$ & exp decay & $5\times 10^{-5}$ & $0.98$ & $1000$  & $(0.8, 0.998)$ \\
\bottomrule
\end{tabular}
\caption{Policy training hyperparameters. All optimisers are Adam with gradient clipping at $\|\nabla\|_2 \le 1$.}
\label{tab:policy_training}
\end{table}

\subsection{Score matching approach}
\label{app:method_score}

\paragraph{Score network architecture.}
For all tasks we use the same backbone architecture, differing only in input/output dimensionalities, the presence of positional encodings, and the choice of inputs to the final output head (detailed below). The backbone is a stack of pre-norm transformer blocks \citep{vaswani_attention_2017} with two prediction heads, one for $\nabla_{y_t} \log p(\y\mid\x)$ and one for $\nabla_{\xi_t} \log p(\y\mid\x)$. The exact configuration is given in \Cref{tab:score_arch} and each component is detailed below. 

\emph{Embedding.} Each $y_t$ and $\xi_t$ is first standardised then separately embedded with a random Fourier features layer \citep{tancik_fourier_2020} (with $\sigma=10$) into $64$-dimensional features each. These are concatenated and passed through a $2$-layer MLP with $[192, 128]$ hidden units and GeLU activations. A linear projection maps these to the $256$-dim model space, giving a sequence of $T$ per-timestep embeddings.

\emph{Positional encodings.} When positional encodings are required (see below), a sinusoidal encoding \citep{vaswani_attention_2017} is added to the per-timestep embeddings.

\emph{Backbone.} A learnable global token is joined to the sequence, making its length $T+1$. The sequence is then processed by $4$ transformer blocks, each consisting of layer norm, multi-head self-attention over the time dimension ($8$ heads), residual, layer norm, a $2$-layer MLP with hidden width $2 \cdot d_{\mathrm{model}}$ and GeLU activations, and a final residual.

\emph{Prediction heads.} Following the main transformer backbone, two independent prediction heads produce the $y$-score and $\xi$-score components. These heads consist of a layer norm followed by a $2$-layer MLP with $d_{\mathrm{model}}$ and GeLU activations. The prediction heads receive different inputs depending on the task, and use a separate ``head embedding'' --- an independent RFF + MLP embedding of the raw $(y_t, \xi_t)$ pair. The inputs are constructed as follows:
\begin{itemize}
    \item \emph{standard inputs} (location finding): the head input is the concatenation of [global token, head embedding at $t$, transformer backbone output at $t$].
    \item \emph{markov window inputs} (dynamical systems): the head input additionally concatenates the head embeddings at $t-1$ and $t+1$ (zero-padded at the boundaries). This Markov window is motivated by the Markovian likelihoods of the dynamical systems since the conditional score at time $t$ has explicit dependence on its one-step neighbours. This setup saves the backbone output from having to reconstruct this information through the attention layers.
\end{itemize}

\emph{Conservative wrapper.} The two-headed transformer described above produces a $d_y + d_\xi$-dimensional vector at each time step. We project this to a scalar potential $\Phi_\psi(\y, \x)$ via a final linear layer, and define the score network as $s_\psi(\y, \x) = \nabla_{(\y, \x)} \Phi_\psi(\y, \x)$, obtained by automatic differentiation. This ensures the learned score is conservative (i.e. it is the gradient of a scalar function), a structural property that the true score $\nabla_{(\y, \x)}\log p(\y\mid\x)$ satisfies. This is a principled design choice which was found to greatly stabilise policy training on the location finding task and therefore was adopted across all tasks. However, we note that computationally this choice considerably increases the burden of evaluating the score network during policy training. 

\emph{Permutation equivariance.} For tasks with an exchangeable likelihood, $p(\y\mid\x)$ is exchangeable with respect to the ordering of $(y_t, \xi_t)$ pairs, so the score is permutation-equivariant in $t$. We bake this into the network by omitting the positional encoding. For the dynamical systems the likelihood is Markovian and the score is not equivariant, therefore we add sinusoidal positional encodings to the per-timestep embeddings.

\begin{table}[h]
\centering
\caption{Score network architecture, shared across all tasks (excluding positional encodings for non-iid tasks and task specific inputs to the output heads). The full network has approximately 3 million parameters.}
\label{tab:score_arch}
\begin{tabular}{lll}
\toprule
Component & Hyperparameter & Value \\
\midrule
RFF embedding & dimension & $64$ (i.e.\ $128$ after concat) \\
              & $\sigma$ & $10.0$ \\
Embedding MLP & hidden widths & $[192, 128]$ \\
              & activation function & GeLU \\
Transformer   & model dim $d_{\mathrm{model}}$ & $256$ \\
              & \# transformer blocks & $4$ \\
              & \# attention heads & $8$ \\
              & MLP hidden width & $2\cdot d_{\mathrm{model}} = 512$ \\
              & activation function & GeLU \\
              & global summary token & yes (learnable) \\
Head embedding & RFF dim & $64$ \\
               & MLP hidden widths & $[192, 128]$ \\
Head MLP       & hidden width & $d_{\mathrm{model}} = 256$ \\
               & activation & GeLU \\
Potential proj. & final linear to scalar $\Phi_\psi$ & yes \\
Score          & $s_\psi = \nabla_{(\y, \x)} \Phi_\psi$ (conservative) & yes \\
\bottomrule
\end{tabular}
\end{table}

\paragraph{Score training.}
The score network is trained on the marginal score-matching loss \Cref{eqn:msm_loss} with the Adam optimiser \citep{kingma_adam_2017} and a warmup-cosine learning-rate schedule that linearly warms the learning rate up over the first $0.1\%$ of training to a peak value, then cosine-decays to a small floor of $10^{-5}$. The standard score-matching loss is augmented with a per-coordinate weighting on the $y$-score component (denoted $\lambda_y$ in \Cref{tab:score_training}) chosen so that the $y$-score loss is comparable in magnitude to the $\xi$-score loss; this is task-dependent because the relative scales of the $y$- and $\xi$-score components differ. Designs for the score-training data distribution are sampled from simple distributions spanning the relevant region of $\xi$-space, discussed in greater detail below. All training hyperparameters are summarised in \Cref{tab:score_training}. The number of training iterations is dictated for each experiment by the allowed NLE budget, which we describe in Appendix \ref{app:nle}.

\begin{table}[h]
\centering
\begin{tabular}{lccccc}
\toprule
Task & batch size & peak LR & grad-norm clip & $\lambda_y$ \\
\midrule
Location finding   & $1024$ & $2 \times 10^{-4}$ & $600$  & $30.0$   \\
Stochastic pendulum & $1024$ & $2 \times 10^{-4}$ & $5{,}000$  & $0.0125$ \\
Cart-pole         & $1024$ & $2 \times 10^{-4}$ & $600$  & $0.004$   \\
Double link        & $1024$ & $2 \times 10^{-4}$ & $10{,}000$ & $0.02$   \\
Gravimetry       & $1024$ & $2 \times 10^{-4}$ & $5{,}000$ & $10000$  \\
\bottomrule
\end{tabular}
\caption{Score network training hyperparameters. Optimiser is Adam with default $(\beta_1, \beta_2) = (0.9, 0.999)$. Learning rate uses linear warmup over $0.1\%$ of training iterations to the peak value, then cosine decay to $10^{-5}$.}
\label{tab:score_training}
\end{table}

\paragraph{Score-based policy training.}
Once the score network is trained, the policy is optimised using the EIG gradient estimator of \Cref{eqn:adaptive_design_grad_prop}, with the intractable score terms replaced by $s_\psi$. The policy network and its optimiser settings are shared with the baselines and described in \Cref{app:policy_nets}.

\subsubsection{Design sampling distributions}
\label{app:design_sampler_choices}

As described in \Cref{sec:method}, our marginal score matching approach samples designs from a distribution $q(\xi)$ which we are free to choose. In the following we describe the distributions used in each experiment, while in Appendix \ref{app:design_sampler_abl} we ablate a wider selection of sampling distributions to illustrate the influence of this choice. In each case the chosen samplers reflect our a-priori beliefs about the regions of design space which may be likely under an optimal policy.

\paragraph{Location finding.} Here we use a temporally correlated Gaussian distribution with uniform hyperpriors on temporal correlation and random scale. Specifically, for each sampled trajectory we first draw a scale $\sigma \sim \mathcal{U}(0.2, 5.0)$ and a temporal correlation coefficient $\rho \sim \mathcal{U}(0.7, 1.0)$, and then sample the full design sequence $\xi = (\xi_1, \dots, \xi_T)$, with each $\xi_t \in \mathbb{R}^d$, from the zero-mean Gaussian
\begin{equation}
    \xi \sim \mathcal{N}\!\left(0,\ \sigma^2\, K_\rho \otimes I_d\right), \qquad [K_\rho]_{s,t} = \rho^{\lvert s - t\rvert},
\end{equation}
where $K_\rho \in \mathbb{R}^{T \times T}$ is an AR(1) correlation matrix over the $T$ experiment steps and the $d$ spatial coordinates of each design are sampled independently. This construction gives designs that are marginally $\mathcal{N}(0, \sigma^2 I_d)$ but whose successive locations are smoothly correlated in time: $\rho \to 1$ yields near-constant trajectories while smaller $\rho$ produces more rapidly varying ones.

\paragraph{Dynamical systems.} Here we draw design trajectories from an equal mixture of two distributions. Firstly, a simple i.i.d. uniform distribution on the design space $[-B, B]$. Secondly, a `pseudo-random' design distribution which concentrates designs on the boundary by randomly switching between each boundary according to a Poisson process. Specifically, for each trajectory and each design dimension we sample a sign process $s_t \in \{-1, +1\}$: the initial sign $s_1$ is drawn uniformly, and at each subsequent step the sign flips with probability $1 - e^{-\lambda}$, where the switch rate is drawn per trajectory as $\lambda \sim \mathcal{U}(0.02, 0.1)$. This is a discrete-time analogue of a Poisson switching process, so on average the design remains pinned to one boundary for a geometrically distributed number of steps before flipping to the other. In practice we apply a small jitter to sit just inside the boundary by adding standard Gaussian noise around a nominal boundary level $u_B$ for which $\tanh(u_B) \approx 1$. In other words, we sample $\xi_t = B \tanh(s_t u_B + \varepsilon)$ where $\varepsilon$ is Gaussian noise and $s_t$ is the sign process.

\paragraph{Gravimetry.} Here we sample designs independently from a Gaussian with trajectory-level random scale $\sigma \sim \mathcal{U}[0.3, 1.2]$. Designs are then mapped to the physical space using the output activation $\xi_t = 40 u_t$, similarly to the policy output activation detailed in Appendix \ref{app:policy_nets}.

\subsection{Baselines}
\label{app:method_baselines}

\paragraph{PCE.}
PCE \citep{foster_unified_2020} uses the same policy architecture and training schedule as our score-based approach (see \Cref{tab:policy_training}) but replaces the score-based gradient estimator with the sequential PCE estimator \citep{foster_deep_2021}. The number of contrastive samples for each task is dictated by the allowed NLE budget which we describe in Appendix \ref{app:nle}. 

\paragraph{Pre-trained variational marginal (VarMarg).}
The variational marginal approach \citep{foster_variational_2019} replaces the intractable $\log p(\y\mid\x)$ in the MI form of the EIG with a learned approximation $\log q_\eta(\y\mid\x)$, providing an upper-bound surrogate. The variational approximation $q_\eta$ is a Masked Autoregressive Flow \citep{papamakarios_masked_2017} with rational-quadratic-spline transforms \citep{durkan_neural_2019}: $5$ MAF layers, conditioner networks of depth $2$ and width $80$, and $8$ spline bins on an interval of $[-4, 4]$. The base distribution is a standard Normal of dimension $T(d_y + d_\xi)$, and the flow is trained by maximum likelihood on samples from the model. We train $q_\eta$ using a warmup-cosine-decay learning-rate schedule with peak $10^{-3}$ and floor $10^{-6}$. We sample designs from the same design sampling distributions as the score network (Appendix \ref{app:design_sampler_choices}) and we model the joint $p(\y,\x)$ then subtract the analytic marginal design score to obtain the required conditional. The policy is then trained with $\log q_\eta(\y \mid \x)$ providing an approximation of the intractable marginal. The policy training hyperparameters are identical to those in \Cref{tab:policy_training} and the number of training iterations is dictated by the allowed NLE budget as per Appendix \ref{app:nle}.

\paragraph{Pre-trained variational posterior (VarPost).} In this approach, we pre-train a variational approximation of the posterior distribution $p(\theta|h_T)$ and use the resulting approximation to train the policy by maximising the Barber-Agakov lower bound on $\mathcal{I}_T$ \citep{barber_information_2003}. Again we sample designs from the same design sampling distributions as the score network (Appendix \ref{app:design_sampler_choices}).

\paragraph{Joint variational posterior (JointVarPost).}
The joint variational posterior approach \citep{foster_unified_2020} co-optimises the variational posterior $q_\eta(\theta \mid \y, \x)$ and the policy by maximising the Barber-Agakov lower bound on the EIG.

The variational posterior $q_\eta(z\mid \y, \x)$ is shared between each of the variational approaches. It is a transformer-encoded conditional mixture of Gaussians where the transformer backbone mimics the architecture of the score networks as closely as possible to provide a fair comparison. Specifically, it is identical to the score transformer network presented in Appendix \ref{app:method_score} up until the prediction heads, where here we use per-component heads to predict the means and diagonal log-scales of the $K$ mixture components and the mixing logits. We use $K=10$ components. The network is specialised further in the location finding task where we canonicalise the hidden source ordering by sorting on their first coordinate to maintain identifiability. Furthermore, in the dynamical systems tasks we use positional encodings to capture the permutation-aware structure of the conditioning observations. 

In each variational approach, we train the posterior with batch size $1024$ and one posterior sample per data point per step. The posterior optimiser is Adam with a warmup-cosine schedule (peak $10^{-3}$, floor $10^{-6}$, warmup over $0.1\%$ of training). In the joint approach, the policy optimiser is Adam with a cosine-decay schedule from $10^{-5}$ to $5\times 10^{-6}$.

\paragraph{\iosmc.}
For \iosmc we use the authors' public Jax implementation \citep{iqbal_nesting_2024}. We only modified the diffusion matrices from their set-up in order to avoid the extremely high EIG settings where policy performance is hard to reliably validate. In order to compensate for the lower reward signal in this regime, we experimented with increasing the tempering parameter and reducing the strength of the slew rate penalty on designs, however neither of these provided any noticeable improvements. The hyperparameters used are therefore identical to the original implementation and we varied the number of particles used in order to respect the fixed NLE budget, which is described in Appendix \ref{app:nle}.

\subsection{Evaluation Bounds}
\label{app:evaluation}

For each trained policy we report upper and lower bounds on the true total EIG $\mathcal{I}_T(\phi)$. These bounds are introduced by \citet{foster_deep_2021} and are computed as follows. Let $\theta_0 \sim p(\theta)$ and $\y \sim p_\phi(\y\mid \theta_0)$ denote an outer sample from the model under the policy, and let $\theta_{1:M} \stackrel{\text{i.i.d.}}{\sim} p(\theta)$ denote $M$ independent contrastive samples from the prior.

\paragraph{sNMC upper bound.}
The sequential NMC estimator of \citet{rainforth_nesting_2018, foster_deep_2021} gives the upper bound
\begin{equation}
\label{eqn:snmc}
    \mathcal{I}_{T,M}^{\,\mathrm{sNMC}}(\phi) = \E\left\{ \log\frac{p(\y\mid \theta_0, \x)}{\frac{1}{M}\sum_{m=1}^M p(\y\mid \theta_m, \x)} \right\} \ge \mathcal{I}_T(\phi),
\end{equation}
which is asymptotically tight as $M \to \infty$.

\paragraph{sPCE lower bound.}
The sequential prior contrastive estimator of \citet{foster_deep_2021} gives the lower bound
\begin{equation}
\label{eqn:spce}
    \mathcal{I}_{T,M}^{\,\mathrm{sPCE}}(\phi) = \E\left\{ \log\frac{p(\y\mid \theta_0, \x)}{\frac{1}{M+1}\sum_{m=0}^M p(\y\mid \theta_m, \x)} \right\} \le \mathcal{I}_T(\phi),
\end{equation}
which is asymptotically tight as $M \to \infty$ and bounded above by $\log(M+1)$.

\paragraph{Sample sizes.}
Both bounds are estimated with $M = 1 \times 10^6$ inner (contrastive) samples and $N = 2048$ outer samples for every (policy, task) pair, with the exception of the standard location finding task and the gravimetry tasks, where $M=5 \times 10^6$ contrastive samples were used. With $N=2048$ samples, the standard error of the bound estimates was found to be below the standard error over randomness in the choice of final policy.

\subsection{NLE Calculations}
\label{app:nle_calcs}

\begin{table}[t]
\centering
\begin{tabular}{lll}
\toprule
Method & Upfront & Per-policy \\
\midrule
Score-based         & $K_U B_U T$  & $K_P N T$        \\
PCE           & ---          & $K_P N (M+1) T$  \\
VarMarg & $K_U B_U T$    & $K_P N T$        \\
VarPost & $K_U B_U T$    & $K_P N T$        \\
JointVarPost & ---         & $K_J B_J T$          \\
\iosmc (lower)      & ---          & $K S (NMT + (N{-}1)MT) + NMT$ \\
\iosmc (upper)      & ---          & $K S (NMT + (N{-}1)(MT + MR T(T{+}1)/2)) + NMT + NMR T(T{+}1)/2$ \\
\bottomrule
\end{tabular}
\caption{Per-policy and total-budget NLE formulas for each method. ``Upfront'' is paid once and amortises over $P$ policies. \iosmc costs are stated as lower/upper bounds spanning the IBIS resampling regime; the realised cost depends on the rate at which the degeneracy criterion triggers.}
\label{tab:nle_summary}
\end{table}

To ensure a fair comparison and highlight the benefit of pre-training the score network in our approach, we compare all methods under a fixed budget of \emph{number of likelihood evaluations} (NLE) consumed during training. We define one NLE as a single evaluation of the conditional likelihood $p(y_t \mid \theta, \xi_t, h_{t-1})$, so an evaluation of the joint log-likelihood over the full experiment costs $T$ NLE. We use NLE as a natural cost metric because it provides a comparable measure of how much information each method consumes from the model. For implicit methods (the variational posterior approaches), these methods do not evaluate the model likelihood but do require samples from it, therefore we use the number of samples as an equivalent proxy. 

We use the following common notation: $K_P$ is the number of policy gradient steps, $N$ the policy training outer batch size, $T$ the experiment length, $P$ the number of policies trained, $K_U, B_U$ the number of iterations and batch size for upfront training (applies to score, VarMarg and VarPost methods), $M$ the number of contrastive samples per outer sample for PCE, and $K_J, B_J$ the joint iterations and batch size for the joint variational posterior training. The total cost across $P$ policies is $\text{upfront} + P \cdot \text{per-policy}$. We summarise the calculations in \Cref{tab:nle_summary}.

\paragraph{Score-based (and pre-trained variational).}
Our score-based approach incurs NLE in two phases. The upfront variational approaches share an identical cost structure. The upfront score-training phase performs $K_U$ iterations with batch size $B_U$; each sample requires one evaluation of the per-time-step grad-log-likelihood in the marginal score matching objective (\Cref{eqn:msm_loss}), which we count as $T$ NLE per sample. The upfront cost is therefore $K_U B_U T$. The policy-training phase performs $K_P$ iterations with batch size $N$; each sample requires one evaluation of the joint log-likelihood to compute the conditional-likelihood gradient term of \Cref{eqn:adaptive_design_grad_prop}, costing $T$ NLE. The per-policy cost is therefore $K_P N T$, and total cost across $P$ policies is
\begin{equation}
    C_{\text{score}}(P) = K_U B_U T + P \cdot K_P N T.
\end{equation}
The upfront score-training cost is paid once and amortises across all $P$ policies, giving a key computational advantage when $P$ grows.

\paragraph{PCE.}
Using the PCE \citep{foster_deep_2021}, the policy is trained for $K_P$ iterations with batch size $N$. Each sample requires one evaluation of the joint log-likelihood under the realised $\theta$ ($T$ NLE), plus an inner Monte-Carlo estimate of $p(\y\mid\x)$ using $M$ samples from the prior ($MT$ NLE). The per-policy cost is therefore $K_P N (M+1) T$ and the total is
\begin{equation}
    C_{\text{PCE}}(P) = P \cdot K_P N (M+1) T.
\end{equation}

\paragraph{Co-trained variational posterior.}
The variational-posterior approach \citep{foster_unified_2020} co-trains a posterior $q_\eta(\theta\mid\y, \x)$ and the policy on the Barber-Agakov lower bound. We charge $T$ NLE per outer sample so that joint training for $K$ iterations with batch size $B$ costs
\begin{equation}
    C_{\text{var-post}}(P) = P \cdot K_J B_J T.
\end{equation}
The cost is per-policy because the posterior is specialised to the current policy and is not re-used across restarts.

\paragraph{\iosmc.}
\iosmc \citep{iqbal_nesting_2024} implements a particle filter (IBIS) within a particle filter (CSMC) within Markovian score climbing. Working from the outermost loop inwards: score climbing performs $K$ gradient-based maximisation iterations on the policy parameters, each of which calls \iosmc once to sample $N$ trajectories from a reward-tilted distribution. Note that the policy-parameter gradient is taken through the policy density $\nabla_\phi \log \prod_{t=1}^T \pi_\phi(\xi_t \mid z_{0:t-1})$, not through the dynamics likelihood, so the outer optimisation itself contributes no likelihood evaluations.

The inner \iosmc algorithm uses IBIS to approximate the posterior over $\theta$ at each step with $M$ particles. Each IBIS step costs at minimum $M$ NLE. If a degeneracy criterion is met at step $t$, $R$ move-resample steps cost an additional $MRt$ NLE. The per-step IBIS cost is therefore between $M$ and $M(1 + Rt)$. Particle re-weighting in the CSMC kernel itself costs $M$ NLE per trajectory per step. Summing over $T$ steps and $N$ trajectories (IBIS is run for $N{-}1$ trajectories since the reference trajectory is inherited), and accounting for the one-off cost of initialising the reference trajectory via unconditional SMC, the total cost of \iosmc with $S$ CSMC kernel applications per gradient step satisfies
\begin{align}
    C_{\text{IO-SMC}^2}^{\text{lower}}(P) &= P \cdot \big[K S (NMT + (N{-}1) MT) + NMT\big], \\
    C_{\text{IO-SMC}^2}^{\text{upper}}(P) &= P \cdot \Big[K S \big(NMT + (N{-}1)(MT + MR T(T{+}1)/2)\big) + NMT + NMRT(T{+}1)/2\Big].
\end{align}
The realised cost falls between these bounds. In our experiments we track the actual number of evaluations and report these additionally. The original \iosmc paper uses $S = 1$ and $R = 3$.

\begin{table}[t]
\centering
\caption{Location Finding budget configurations. All methods share a fixed NLE budget of exactly $6.144\times10^{9}$; PCE$^{*}$ uses $10\times$ this budget ($6.144\times10^{10}$) reflecting more typical values required for this method. Shared fixed values (across all rows): experiment length $T=30$, policy-training iterations $K_P=10^{4}$, policy batch size $N=1024$, and upfront/joint batch size $B_U=B_J=1024$. We tabulate only the quantities varied to meet the budget: the number of policies $P$, the upfront/joint training iterations $K_U$/$K_J$, and the number of PCE contrastive samples $M$.}
\label{tab:nle_config_lf}
\begin{tabular}{lrrr}
\toprule
Method & $P$ & $K_U\,/\,K_J$ & $M$ \\
\midrule
Score-based (plus VarMarg, VarPost)   & 1 & 190{,}000 & --- \\
Score-based (plus VarMarg, VarPost)  & 5 & 150{,}000 & --- \\
\midrule
PCE           & 1 & ---       & 19  \\
PCE           & 5 & ---       & 3   \\
PCE$^{*}$     & 1 & ---       & 199 \\
\midrule
JointVarPost  & 1 & 200{,}000 & --- \\
JointVarPost  & 5 & 40{,}000  & --- \\
\bottomrule
\end{tabular}
\end{table}

\begin{table}[h]
\centering
\caption{Dynamical Systems configurations. All methods share a fixed NLE budget of exactly $1.024\times10^{10}$. Shared fixed values (across all rows): experiment length $T=50$, policy-training iterations $K_P=500$, policy batch size $N=1024$, and upfront/joint batch size $B_U=B_J=1024$. We tabulate only the varied quantities: the number of policies $P$, the upfront/joint training iterations $K_U$/$K_J$, and the number of PCE contrastive samples $M$. Realised \iosmc configurations are reported separately in \Cref{tab:nle_iosmc}.}
\label{tab:nle_config_ds}
\begin{tabular}{lrrr}
\toprule
Method & $P$ & $K_U\,/\,K_J$ & $M$ \\
\midrule
Score-based (plus VarMarg, VarPost)   & 1  & 199{,}500 & --- \\
Score-based (plus VarMarg, VarPost)   & 50 & 175{,}000 & --- \\
\midrule
PCE           & 1  & ---       & 399 \\
PCE           & 50 & ---       & 7   \\
\midrule
JointVarPost  & 1  & 200{,}000 & --- \\
JointVarPost  & 50 & 4{,}000   & --- \\
\bottomrule
\end{tabular}
\end{table}

\begin{table}[t]
\centering
\caption{Gravimetry configurations. All methods share a fixed NLE budget of $4.096\times10^{9}$; PCE$^{*}$ uses $10\times$ this budget ($4.096\times10^{10}$). Shared fixed values (across all rows): experiment length $T=20$, policy-training iterations $K_P=10^{4}$, policy batch size $N=1024$, and upfront/joint batch size $B_U=B_J=1024$. We tabulate only the varied quantities: the number of policies $P$, the upfront/joint training iterations $K_U$/$K_J$, and the number of PCE contrastive samples $M$.}
\label{tab:nle_config_grav}
\begin{tabular}{lrrr}
\toprule
Method & $P$ & $K_U\,/\,K_J$ & $M$ \\
\midrule
Score-based (plus VarMarg, VarPost)   & 1 & 190{,}000 & --- \\
Score-based (plus VarMarg, VarPost)   & 5 & 150{,}000 & --- \\
\midrule
PCE           & 1 & ---       & 19  \\
PCE           & 5 & ---       & 3   \\
PCE$^{*}$     & 1 & ---       & 199 \\
\midrule
JointVarPost  & 1 & 200{,}000 & --- \\
JointVarPost  & 5 & 40{,}000  & --- \\
\bottomrule
\end{tabular}
\end{table}

\begin{table}[t]
\centering
\caption{\iosmc NLE bounds on Dynamical Systems experiments. Here we report the configuration used and the upper and lower bounds on NLE cost. The configuration was chosen to capture the target NLE value of $1.024 \times 10^{10}$ allowed to each other method. We then report the realised NLE value on each experiment. We did not use a $P=50$ variant for \iosmc as the wall-clock time required to train and evaluate the required number of policies using the \iosmc implementation was prohibitive.}
\label{tab:nle_iosmc}
\begin{tabular}{lrr}
\toprule
$(K, N, M, R, S)$ & lower & upper  \\
\midrule
$(25, 1024, 512, 3, 1)$ & $1.34 \times 10^{9}$ & $5.34 \times 10^{10}$ \\
\bottomrule
\end{tabular}

\vspace{1em}

\begin{tabular}{lrrr}
    \toprule
    & Stochastic Pendulum & Cart-Pole & Double Pendulum \\
    \midrule
    Realised NLE & $9.9 \times 10^9 \pm 5.7 \times 10^7$ & $1.5 \times 10^{10} \pm 3.6 \times 10^7$ & $1.8 \times 10^{10} \pm 7.8 \times 10^7$ \\
    \bottomrule
\end{tabular}
\end{table}

\subsection{NLE Values}
\label{app:nle}
In the previous subsection, we detailed the NLE cost structure of each method. In our experimental protocol, all methods were assigned a fixed NLE budget. We presented the results for each method by training either one or multiple policies $P$. For $P>1$ we select the best performing policy based on an initial estimate of its PCE lower bound, and report an independent re-estimate of the EIG bounds for that selected policy. Depending on $P$, we adjusted the number of upfront/joint training iterations $K_U, K_J$ and the number of contrastive samples $M$ for each method (similarly the number of particles $M$ and trajectories $N$ in \iosmc) in order to respect the fixed NLE budget available. This is representative of the procedure one might undertake when selecting a policy network by testing multiple architectures/hyperparameters or by training multiple policies to avoid local optima. We provide full details of the NLE budget and the configurations used to achieve these budgets in \Cref{tab:nle_config_lf,tab:nle_config_ds,tab:nle_config_grav,tab:nle_iosmc}.

\section{Ablation Studies, Further Results and Insights}
\label{app:ablations}

\subsection{MLMC Results}
\label{app:mlmc_results}

We implemented the MLMC approach of \citet{goda_unbiased_2022} and tested this on the standard location finding experiment. MLMC uses a debiasing scheme which randomises the number of inner samples to provide an unbiased estimate of the EIG gradient. In our experiments we found that the resulting gradients are high variance, therefore we increased the number of outer samples per gradient estimate and trained the policy network for longer to allow proper convergence. We obtained a \textbf{lower bound of $7.21 \pm 0.40$ and an upper bound of $7.27 \pm 0.40$ on the $d=2, K=2$ location finding task}. Unfortunately, the implementation took $>24$ hours to train a single policy, therefore we were unable to provide results across all our experimental settings. We provide the experimental settings and approximate NLE cost used for these results in \Cref{tab:mlmc_config}. Note that the configuration we tested far exceeds the NLE cost of the budget for the main results, yet the MLMC performance was still significantly worse.

\paragraph{NLE cost of MLMC.} Each outer sample $n=1,\dots,N$ is assigned a random level $L_n$ drawn from a geometric distribution,
\begin{equation}
  \Pr(L = l) = (1 - 2^{-\tau})\,(2^{-\tau})^{l},
  \qquad l = 0, 1, 2, \dots,
\end{equation}
and the number of inner (nested) samples at level $l$ is $M_0 2^{l}$. The expected number of inner samples per outer sample is therefore
\begin{equation}
  \mathbb{E}\!\left[M_0 2^{L}\right]
  = M_0 (1 - 2^{-\tau}) \sum_{l=0}^{\infty} \left(2^{1-\tau}\right)^{l}
  = M_0 \,\frac{1 - 2^{-\tau}}{1 - 2^{1-\tau}},
\end{equation}
which is finite only for $\tau > 1$. We used $\tau = 1.1$, for which this is $\approx 7.96 M_0$. This gives a total theoretical expected NLE cost of $7.96 T K_P N M_0$. 

To compile and parallelise the loop we require a static number of inner samples. We therefore cap the level $L_n$ at $l_\text{max}=12$, which clips a total mass of at most 0.01\% of the tail. The actual realised compute cost is therefore $T N K_P M_0 2^{l_\text{max}}$, in other words equivalent to PCE with $2^{12}M_0 = 4096M_0$ inner samples.

\begin{table}
    \centering
    \caption{Configuration and NLE values for MLMC on the standard location finding experiment. Theoretical NLE gives an expected NLE based on the expected number of inner samples per outer sample in the randomised scheme. Realised NLE reports the actual NLE required due to padding to allow parallelisation.}
    \label{tab:mlmc_config}
    \begin{tabular}{lr}
        \toprule
        $N$ & 2000 \\
        $M_0$ & 1 \\
        $K_P$ & 50000 \\
        $\tau$ & 1.1 \\
        Theoretical NLE & $2.4 \times 10^{10}$ \\
        Realised NLE & $2.0 \times 10^{13}$ \\ 
        \bottomrule
    \end{tabular}
\end{table}

\subsection{Gravimetry Results}
\label{app:extra_gravimetry}

We consider an additional experimental design task inspired by the gravity methods described in \citet{Telford_Geldart_Sheriff_1990} with which one infers the unknown location, depth and strength of an underground void based on measurements of local gravitational anomalies. We consider a one-dimensional search space, where designs $\xi_t \in \mathbb{R}$ are placed on a line and a noisy measurement of the gravitational deviation is observed. This task can be seen as a more challenging development of the source location finding task. In particular, the signal strength---a function of the size and relative density of the underground void---is now unknown and the depth and signal strength can only be independently identified by measurements adjacent to but not directly above the source. We place $T=20$ designs and train an adaptive TNP-style policy network~\citep{pmlr-v162-nguyen22b}. Results are presented in \Cref{tab:grav_1d_results} where again we observe that \scorebed performs on par with PCE and slightly worse than the variational approaches.

\begin{table}
    \centering
    \begin{tabular}{lrr}
        \toprule
        Method & Lower Bound & Upper Bound \\
        \midrule
        \scorebed (P=1) & 8.09 $\pm$ 0.09 & 9.31 $\pm$ 0.11 \\
        \scorebed (P=5) & 8.42 $\pm$ 0.16 & 9.49 $\pm$ 0.27 \\
        PCE (P=1) & 8.43 $\pm$ 0.06 & 8.78 $\pm$ 0.09 \\
        PCE (P=5) & 8.37 $\pm$ 0.06 & 8.85 $\pm$ 0.12 \\
        JointVarPost (P=1) & \textbf{8.86 $\pm$ 0.04} & 10.20 $\pm$ 0.15 \\
        JointVarPost (P=5) & \textbf{8.80 $\pm$ 0.04} & 10.54 $\pm$ 0.48 \\
        VarMarg (P=1) & 7.51 $\pm$ 0.13 & 8.19 $\pm$ 0.15 \\
        VarMarg (P=5) & 7.65 $\pm$ 0.06 & 8.55 $\pm$ 0.14 \\
        VarPost (P=1) & \textbf{8.85 $\pm$ 0.05} & 9.72 $\pm$ 0.07 \\
        VarPost (P=5) & \textbf{8.87 $\pm$ 0.04} & 9.91 $\pm$ 0.15 \\
        \bottomrule
    \end{tabular}
    \caption{Results on the 1D gravimetry task. Each approach is allowed an NLE budget of $4.096 \times 10^{9}$. Results are mean ± one standard error over policy repeats with bold indicating statistically significant best lower bound at 95\% confidence level.}
    \label{tab:grav_1d_results}
\end{table}

\subsection{Choice of design sampling distribution}
\label{app:design_sampler_abl}

\begin{table}
    \centering
    \begin{tabular}{lrr}
        \toprule
        Sampler & Lower Bound & Upper Bound \\
        \midrule
        Gaussian & 2.86 $\pm$ 0.26 & 2.86 $\pm$ 0.26 \\    
        Random scale Gaussian & \textbf{10.84 $\pm$ 0.08} & 10.97 $\pm$ 0.09 \\
        Temporal correlation & 10.81 $\pm$ 0.06 & 10.95 $\pm$ 0.07 \\
        Low rank projection & \textbf{11.01 $\pm$ 0.02} & 11.18 $\pm$ 0.02 \\
        \bottomrule
    \end{tabular}
    \caption{Ablation results for design sampler choice on the $d=2, K=2$ location finding task. Final policies are obtained using the $P=1$ protocol. Results present mean ± one standard error with bold indicating statistically best lower bound at the 95\% confidence level.}
    \label{tab:abl_lf_design_sampler}
\end{table}

% \begin{figure}
%     \centering
%     \includegraphics[width=0.5\textwidth]{figures/camera_ready_appendix/lf_sampler_random_projection_custom.pdf}
%     \caption{Low rank projection samples for location finding task.}
%     \label{fig:lf_low_rank_designs}
% \end{figure}

Here we explore the importance of the design sampling distribution $q(\xi_{1:T})$ used during score network training. We provided details of the sampling distributions used in our main results (see \Cref{sec:experiments}) in Appendix \ref{app:design_sampler_choices}. Here we present additional results using different choices of design sampling distribution.

\paragraph{Location finding.} Here we consider the influence of the design sampling distribution used during score network training for the $d=2, K=2$ location finding task. We consider four different choices of design sampler. Firstly, a standard isotropic Gaussian sampler which simply samples i.i.d. designs $\xi_t \sim \mathcal{N}(0, 1), t=1, ... T$, which represents the simplest possible choice for this experiment. Secondly, we augment the standard Gaussian with a random scale sampled uniformly per trajectory from $\mathcal{U}[0.5, 2]$. Thirdly, we consider a handcrafted design distribution which was specifically created to mimic the behaviour of an optimal DAD policy, which typically exhibits a spiralling behaviour. This was achieved by sampling trajectories as low-rank projections of Gaussian noise, $\xi = \Phi z$, where $\Phi \in \mathbb{R}^{T \times r}$ is a time-basis matrix and $z \in \mathbb{R}^{r \times d}$ are Gaussian coefficients, $z_{ij} \sim \mathcal{N}(0, s^2)$ with an overall scale drawn randomly as $s \sim \mathcal{U}[0.5, 2]$. For each sample the basis $\Phi$ is drawn uniformly from a family of temporal bases --- sinusoidal, Fourier, or spirals --- and the projection rank $r$ is drawn uniformly in $[2, 8]$, so that designs are smooth, low-dimensional curves in time whose complexity varies across samples. Finally, for reference we include the temporally correlated random scale Gaussian which was used in our main results. 

We present the results of the ablation in \Cref{tab:abl_lf_design_sampler}. We find that when the score network is trained on the plain isotropic Gaussian, it fails to train successful policies. The other design sampler choices perform similarly, with marginally best performance from the low rank projection sampler. From our experiments, we found that the critical ingredient for this task is the hyperprior on the sampler scale. This is critical because at initialisation, the policy network emits designs on a scale much smaller than a standard Gaussian and the scale then grows as the policy trains. Including a hyperprior on design scales better covers the design space in the regions that the policy visits during training. The custom low rank projection sampler gives the best performance although it required substantial customisation to do so.

\subsection{Score training stability}
\label{app:score_training_stability}

\begin{figure}[t]
    \centering
    \begin{subfigure}[b]{0.45\textwidth}
        \centering
        \includegraphics[width=\textwidth]{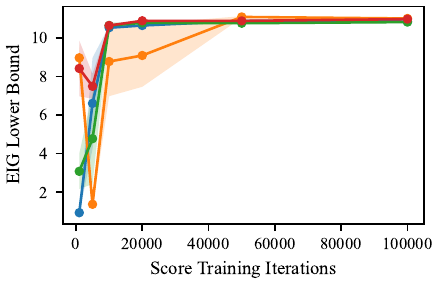}
    \end{subfigure}
    \hspace{2em}
    \begin{subfigure}[b]{0.45\textwidth}
        \centering
        \includegraphics[width=\textwidth]{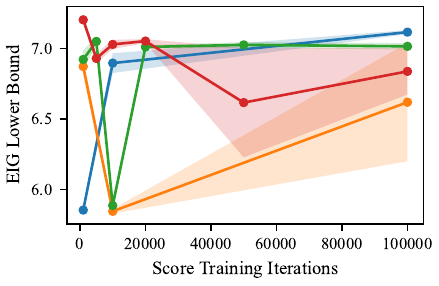}
    \end{subfigure}
    \caption{Stability of \scorebed with respect to score training. Each curve presents the final policy EIG under the $P=1$ protocol against number of score training iterations. Left:  $d=2, K=2$ location finding. Right: cart-pole.}
    \label{fig:score_stability}
\end{figure}

We investigated the stability of \scorebed with respect to score training seed. \Cref{fig:score_stability} plots the final policy EIG under the $P=1$ protocol for multiple score network training seeds on the standard location finding and cart-pole tasks. While \scorebed is clearly very stable on the location finding task, we observe instability early in the score network training, as noted in \Cref{sec:experiments}. In order to compensate for this, we split the allowed score budget between multiple score networks, therefore guarding against this instability. We believe that this instability is more precisely attributed to unstable policy training due to bias in the \scorebed EIG gradient approximation, which we explore in the following section.

\subsection{Additional Bias-Variance Results}
\label{app:bias_var_app_results}

We now extend the discussion from \Cref{sec:exp_bias_var} by exploring the bias of the EIG gradients of \scorebed and PCE on a larger set of tasks. \Cref{fig:eig_mse_vs_N} presents error curves of \scorebed against $N$ for increasing score training iterations on both location finding tasks and two dynamical systems tasks. It is immediately clear that \scorebed achieves considerably lower bias than PCE on both location finding tasks, while \scorebed carries greater bias on both the stochastic pendulum and cart-pole tasks compared with PCE. As per our bias--variance decomposition and discussions in \Cref{sec:exp_bias_var}, we believe that gradient bias is a key factor in the quality of policy training as large bias can throw off policy training by pushing the policy in suboptimal directions. Indeed, the scatter plots in \Cref{fig:eig_scatters_all} show the correlation between EIG gradient bias and final policy performance, illustrating that lower bias is associated with stable policy performance. We see that the bias of \scorebed on the dynamical systems tasks may be on the verge of stable performance, but by selecting policies between multiple score training seeds in the $P>1$ protocol, we are still able to avoid this instability. 

In order to investigate the cause of this bias on the dynamical systems tasks, we explored correlation between score MSE and the conditioning of the Jacobian $J_\xi(y) = \frac{\partial y}{\partial \xi}$ appearing in the gradient bias (cf. \cref{eqn:app_data_error_cauchy_schwarz}), see \Cref{fig:jacobian_all}. We found that in the trained (final) policy regime, trajectories have significantly worse Jacobian conditioning and there is apparent correlation between score MSE and Jacobian conditioning. This is unsurprising since the non-Markovian property of the dynamical systems models makes them particularly sensitive to perturbations early in the trajectory, and it suggests the bound on EIG gradient error in \cref{eqn:app_data_error_cauchy_schwarz}--- equivalently the constant $L_y$ in \cref{eqn:component_separate_err_bounds}---may be such that the score networks are required to be extremely accurate on this task to provide sufficiently low EIG gradient bias.

\begin{figure}[h]
    \centering
    \begin{subfigure}[b]{0.245\textwidth}
        \centering
        \includegraphics[width=\textwidth]{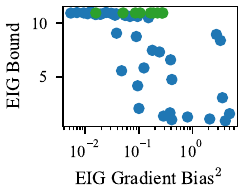}
    \end{subfigure}
    \hfill
    \begin{subfigure}[b]{0.245\textwidth}
        \centering
        \includegraphics[width=\textwidth]{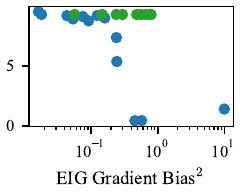}
    \end{subfigure}
    \hfill
    \begin{subfigure}[b]{0.245\textwidth}
        \centering
        \includegraphics[width=\textwidth]{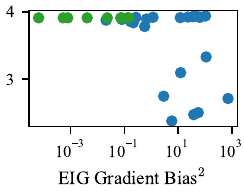}
    \end{subfigure}
    \hfill
    \begin{subfigure}[b]{0.245\textwidth}
        \centering
        \includegraphics[width=\textwidth]{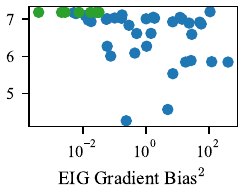}
    \end{subfigure}
    \caption{Scatter plots showing correlation between EIG gradient bias and final policy performance, with points representing different score training budgets/seeds (blue) and different numbers of PCE inner samples (green). From left to right: location finding ($d=2, K=2$), location finding ($d=3, K=10$), stochastic pendulum, cart-pole. }
    \label{fig:eig_scatters_all}
\end{figure}

\begin{figure}
    \centering
    \begin{subfigure}[b]{0.48\textwidth}
        \centering
        \includegraphics[width=\textwidth]{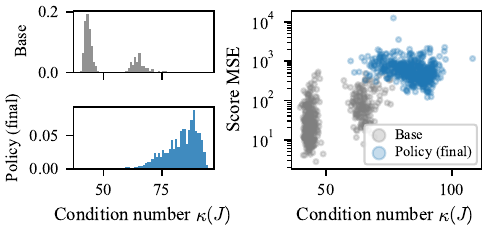}
    \end{subfigure}
    \hfill
    \begin{subfigure}[b]{0.48\textwidth}
        \centering
        \includegraphics[width=\textwidth]{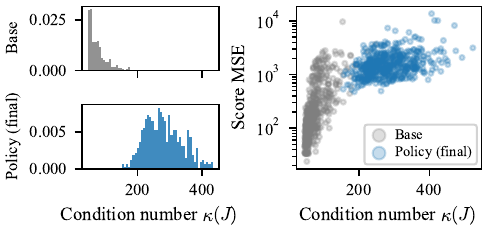}
    \end{subfigure}
    \caption{Histograms showing the distribution of Jacobian condition number between the score training distribution and the distribution induced by the final policy network. Scatter plots show correlation between score MSE and Jacobian condition number. Left: stochastic pendulum. Right: cart-pole.}
    \label{fig:jacobian_all}
\end{figure}

\begin{figure}
    \centering
    \begin{subfigure}[b]{0.4\textwidth}
        \centering
        \includegraphics[width=\textwidth]{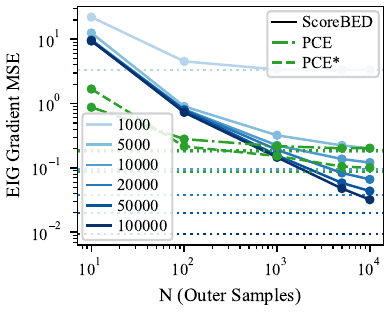}
    \end{subfigure}
    \hspace{2em}
    \begin{subfigure}[b]{0.4\textwidth}
        \centering
        \includegraphics[width=\textwidth]{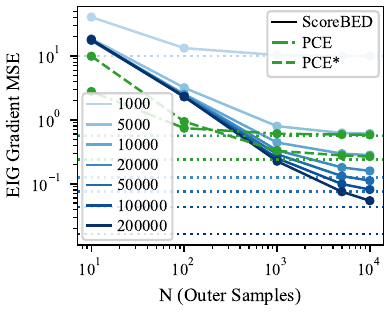}
    \end{subfigure}

    \vspace{\baselineskip}

    \begin{subfigure}[b]{0.4\textwidth}
        \centering
        \includegraphics[width=\textwidth]{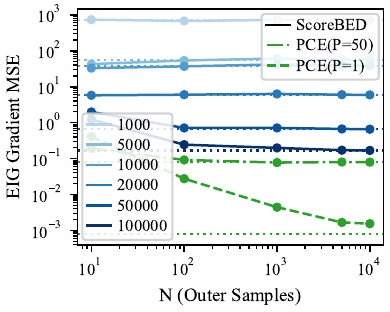}
    \end{subfigure}
    \hspace{2em}
    \begin{subfigure}[b]{0.4\textwidth}
        \centering
        \includegraphics[width=\textwidth]{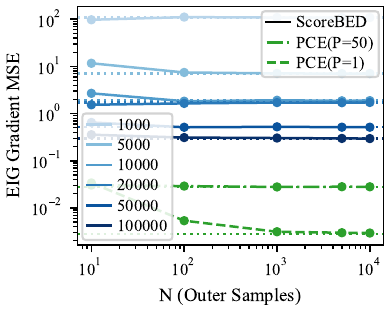}
    \end{subfigure}
    \caption{EIG gradient MSE curves for \scorebed and PCE. PCE uses $M=19$ and PCE$^{\ast}$ uses $M=199$ inner samples on the location finding tasks, while PCE ($P=50$) uses $M=7$ and PCE ($P=1$) uses $M=399$ on the dynamical systems tasks. These values are chosen according to the fixed NLE budgets for each task as per Appendix \ref{app:nle}. From top left to bottom right: location finding ($d=2, K=2$), location finding ($d=3, K=10$), stochastic pendulum, cart-pole.}
    \label{fig:eig_mse_vs_N}
\end{figure}

\end{document}

%% file: def.tex
\usepackage{xspace}

\newtheorem{theorem}{Theorem}

\newcommand{\IG}{\mathrm{IG}}

\newcommand{\E}{\mathbb{E}}
\newcommand{\rmd}{\textrm{d}}

\newcommand{\y}{y_{1:T}}
\newcommand{\x}{\xi_{1:T}}
\newcommand{\Tr}{\operatorname{Tr}}

\newcommand{\score}{s_\psi}

\newcommand{\Cov}{\text{Cov}}

\newcommand{\xx}{\xi^\ast}
\newcommand{\yn}{y_{1:N}}

\newcommand{\iosmc}{IO-SMC$^2$\xspace}
\newcommand{\scorebed}{\textsc{ScoreBED}\xspace}

\DeclareMathOperator*{\argmax}{arg\,max}

\DeclarePairedDelimiter\abs{\lvert}{\rvert}%
\DeclarePairedDelimiter\norm{\lVert}{\rVert}%

\makeatletter
\let\oldabs\abs
\def\abs{\@ifstar{\oldabs}{\oldabs*}}
\let\oldnorm\norm
\def\norm{\@ifstar{\oldnorm}{\oldnorm*}}
\makeatother